\documentclass{article}

     \PassOptionsToPackage{numbers, compress}{natbib}

\usepackage[final]{neurips_2022}

\usepackage[utf8]{inputenc} 
\usepackage[T1]{fontenc}    
\usepackage{hyperref}       
\usepackage{url}            
\usepackage{booktabs}       
\usepackage{amsfonts}       
\usepackage{nicefrac}       
\usepackage{microtype}      
\usepackage{xcolor}         

\usepackage{graphicx}
\usepackage{makecell}
\usepackage{caption}
\usepackage{subcaption}
\usepackage{multirow}

\usepackage{colortbl}
\definecolor{bgcolor}{rgb}{0.66,0.88,1.00}

\usepackage{amsfonts,amsmath,amsthm,amssymb,dsfont,etex}
\allowdisplaybreaks
\usepackage{mathrsfs}

\theoremstyle{plain}

\newtheorem{theorem}{Theorem}[section]
\newtheorem{lemma}{Lemma}[section]
\newtheorem{proposition}{Proposition}[section]
\newtheorem{remark}[theorem]{Remark}

\theoremstyle{definition}
\newtheorem{definition}[theorem]{Definition}

\usepackage{threeparttable}

\usepackage{xspace}

\xspaceaddexceptions{[]\{\}}

\newcommand{\dataset}[1]{{\tt #1}\xspace}

\usepackage{hyperref}

\usepackage{thm-restate}

\usepackage{algorithm}
\usepackage{algorithmicx}
\usepackage{algpseudocode}
\algrenewcommand\algorithmicrequire{\textbf{Input:}}
\algrenewcommand\algorithmicensure{\textbf{Output:}}

\usepackage{amsmath,amsfonts,bm}

\newcommand{\FL}{Feldman-Langberg}

\newcommand{\kMeans}{$k$-means}
\newcommand{\costR}{\ensuremath{\mathsf{cost}^R}}
\newcommand{\costC}{\ensuremath{\mathsf{cost}^C}}
\newtheorem{problem}{Problem}

\def\1{\bm{1}}

\def\eps{{\varepsilon}}

\def\vtheta{{\bm{\theta}}}

\def\vxi{{\bm{\xi}}}

\def\vc{{\bm{c}}}

\makeatletter
\newcommand{\ve}{\@ifnextchar\bgroup{\velong}{{\bm{e}}}}
\newcommand{\velong}[1]{{\bm{#1}}}
\makeatother

\def\vg{{\bm{g}}}

\def\vu{{\bm{u}}}

\def\vw{{\bm{w}}}
\def\vx{{\bm{x}}}
\def\vy{{\bm{y}}}

\def\mA{{\bm{A}}}
\def\mB{{\bm{B}}}
\def\mC{{\bm{C}}}
\def\mD{{\bm{D}}}

\def\mI{{\bm{I}}}

\def\mU{{\bm{U}}}

\def\mX{{\bm{X}}}

\def\mZ{{\bm{Z}}}

\DeclareMathAlphabet{\mathsfit}{\encodingdefault}{\sfdefault}{m}{sl}
\SetMathAlphabet{\mathsfit}{bold}{\encodingdefault}{\sfdefault}{bx}{n}

\def\gA{{\mathcal{A}}}

\def\gC{{\mathcal{C}}}

\def\gG{{\mathcal{G}}}

\def\gN{{\mathcal{N}}}

\def\gX{{\mathcal{X}}}
\def\gY{{\mathcal{Y}}}
\def\gZ{{\mathcal{Z}}}

\def\sP{{\mathbb{P}}}

\newcommand{\eat}[1]{}

\DeclareMathOperator*{\argmin}{arg\,min}

\newcommand\numberthis{\addtocounter{equation}{1}\tag{\theequation}}

\newcommand{\norm}[1]{\left\| #1 \right\|}

\newcommand{\normF}[1]{\left\| #1 \right\|_{\mathrm{F}}}

\def\eqref#1{(\ref{#1})}

\newcommand{\R}{\mathbb{R}}

\newcommand{\algname}[1]{{\sf\footnotesize#1}\xspace}

\usepackage{todonotes}

\title{Coresets for Vertical Federated Learning: Regularized Linear Regression and $K$-Means Clustering}

\author{%
  Lingxiao Huang\thanks{Alphabetical order.} \\
  Nanjing University \\
  \texttt{huanglingxiao1990@126.com} \\
  \And
  Zhize Li$^*$\\
  Carnegie Mellon University \\
  \texttt{zhizeli@cmu.edu} \\
  \AND
  Jialin Sun$^*$ \\
  Fudan University \\
  \texttt{sunjl20@fudan.edu.cn} \\
  \And
  Haoyu Zhao$^*$ \\
  Princeton Univeristy \\
  \texttt{haoyu@princeton.edu} \\
}

\begin{document}

\maketitle

\begin{abstract}
    Vertical federated learning (VFL), where data features are stored in multiple parties distributively, is an important area in machine learning. However, the communication complexity for VFL is typically very high. In this paper, we propose a unified framework by constructing \emph{coresets} in a distributed fashion for communication-efficient VFL. We study two important learning tasks in the VFL setting: regularized linear regression and $k$-means clustering, and apply our coreset framework to both problems. We theoretically show that using coresets can drastically alleviate the communication complexity, while nearly maintain the solution quality. Numerical experiments are conducted to corroborate our theoretical findings.
\end{abstract}



\section{Introduction}
\label{sec:introduction}


Federated learning (FL)~\citep{mcmahan2017communication, konevcny2016federated, li2020federated, kairouz2021advances, wang2021field} is a learning framework where multiple clients/parties collaboratively train a machine learning model under the coordination of a central server without exposing their raw data (i.e., each party’s raw data is stored locally and not transferred). 
There are two large categories of FL, horizontal federated learning (HFL) and vertical federated learning (VFL), based on the distribution characteristics of the data. 
In HFL, different parties usually hold different datasets but all datasets share the same features; while in VFL, all parties use the same dataset but different parties hold different subsets of the features (see Figure~\ref{fig:model}).

Compared to HFL, VFL~\citep{yang2019federated, liu2019communication, wei2022vertical} is generally harder and requires more communication: as a single party cannot observe the full features, it requires communication with other parties to compute the loss and the gradient of \textit{a single data}. 
This will result in two potential problems: (i) it may require a huge amount of communication to jointly train the machine learning model when the dataset is large; and (ii) the procedure of VFL transfers the information of local data and may cause privacy leakage.
Most of the VFL literature focus on the privacy issue, and designing secure training procedure for different machine learning models in the VFL setting~\citep{hardy2017private, yang2019federated, weng2020privacy, chen2021secureboost+}. 
However, the communication efficiency of the training procedure in VFL is somewhat underexplored. 
For unsupervised clustering problems, \citet{ding2016kmeans} propose constant approximation schemes for \kMeans\ clustering, and their communication complexity is \emph{linear} in terms of the dataset size.
For linear regression, although the communciaiton complexity can be improved to \emph{sublinear} via sampling, such as SGD-type uniform sampling for the dataset~\citep{liu2019communication, yang2019federated}, the final performance is not comparable to that using the whole dataset.
Thus previous algorithms usually do not scale or perform well to the big data scenarios. \footnote{In our numerical experiments (Section \ref{sec:experiments}), we provide some results to justify this claim.} 
This leads us to consider the following question:

\begin{quote}
    \emph{How to train machine learning models using sublinear communication complexity in terms of the dataset size without sacrificing the performance in the vertical federated learning (VFL) setting?}
\end{quote}

In this paper, we try to answer this question, and our method is based on the notion of \textit{coreset}~\citep{harpeled2004on,feldman2011unified,feldman2013turning}. Roughly speaking, coreset can be viewed as a small data summary of the original dataset, and the machine learning model trained on the coreset performs similarly to the model trained using the full dataset. Therefore, as long as we can obtain a coreset in the VFL setting in a communication-efficient way, we can then run existing algorithms on the coreset instead of the full dataset.

\paragraph{Our contribution} We study the communication-efficient methods for vertical federated learning with an emphasis on scalability, and design a general paradigm through the lens of coreset. Concretely, we have the following key contributions:
\begin{enumerate}
    \item We design a unified framework for coreset construction in the vertical federated learning setting (Section \ref{sec:unified}), which can help reduce the communication complexity (Theorem~\ref{thm:coreset_reduce}).
    \item We study the regularized linear regression  (Definition \ref{def:vertical_regression}) and \kMeans\ (Definition \ref{def:vertical_kmeans}) problems in the VFL setting, and apply our unified coreset construction framework to them. We show that we can get $\eps$-approximation for these two problems using only $o(n)$ sublinear communications under mild conditions, where $n$ is the size of the dataset (Section \ref{sec:vrlr} and \ref{sec:vkmc}).
    \item We conduct numerical experiments to validate our theoretical results. Our numerical experiments corroborate our findings that using coresets can drastically reduce the communication complexity, while maintaining the quality of the solution (Section \ref{sec:experiments}).
    Moreover, compared to uniform sampling, applying our coresets can achieve a better solution with the same or smaller communication complexity.
\end{enumerate}

\subsection{More related works}
\label{sec:related}

\paragraph{Federated learning} Federated learning was introduced by \citet{mcmahan2017communication}, and received increasing attention in recent years. 
There exist many works studied in the horizontal federated learning (HFL) setting, such as algorithms with multiple local update steps~\citep{mcmahan2017communication, das2020improved, karimireddy2020scaffold, gorbunov2021local, mitra2021linear, zhao2021fedpage} .
There are also many algorithms with communication compression~\citep{karimireddy2019error, mishchenko2019distributed, li2020acceleration, li2020unified, gorbunov2021marina, li2021canita, richtarik2021ef21, fatkhullin2021ef21, zhao2021faster, richtarik20223pc,zhao2022beer} 
and algorithms with privacy preserving~\citep{wei2020federated, hu2020personalized, zhao2020local, truex2020ldp,li2022soteriafl}.

\paragraph{Vertical federated learning} 
Due to the difficulties of VFL, people designed VFL algorithms for some particular machine learning models, including linear regression~\citep{liu2019communication, yang2019federated}, logistic regression~\citep{yang2019parallel, yang2019quasi, he2021secure}, gradient boosting trees~\citep{tian2020federboost, cheng2021secureboost, chen2021secureboost+}, and \kMeans~\cite{ding2016kmeans}. 
For $k$-means, \citet{ding2016kmeans} proposed an algorithm that computes the global centers based on the product of local centers, which requires $O(nT)$ communication complexity. For linear regression, \citet{liu2019communication} and \citet{yang2019federated} used uniform sampling to get unbiased gradient estimation and improved the communication efficiency, but the performance may not be good compared to that without sampling. \citet{yang2019quasi} also applied uniform sampling to quasi-Newton algorithm and improved communication complexity for logistic regression.
People also studied other settings in VFL, e.g., how to align the data among different parties~\citep{sun2021vertical}, how to adopt asynchronous training~\citep{chen2020vafl, gu2020privacy}, and how to defend against attacks in VFL~\citep{liu2021rvfr, luo2021feature}.
In this work, we aim to develop communication-efficient algorithms to handle large-scale VFL problems.

\paragraph{Coreset} Coresets have been applied to a large family of problems in machine learning and statistics, including clustering~\cite{feldman2011unified,braverman2016new,huang2020coresets,cohenaddad2021new,cohenaddad2022towards}, regression~\cite{drineas2006sampling,li2013iterative,boutsidis2013near,cohen2015uniform,jubran2019fast,chhaya2020coresets}, low rank approximation~\cite{cohen2017input}, and mixture model~\cite{lucic2017training,huang2020coresetsFR}.
Specifically, \citet{chhaya2020coresets} investigated coreset construction for regularized regression with different norms.
\citet{feldman2011unified,braverman2016new} proposed an importance sampling framework for coreset construction for clustering (including \kMeans). 
The coreset size for \kMeans\ clustering has been improved by several following works~\cite{huang2020coresets,cohenaddad2021new,cohenaddad2022towards} to $\tilde{O}(k\eps^{-4})$, and \citet{cohenaddad2022towards} proved a lower bound of size $\Omega(\eps^{-2}k)$.
Due to the mergable property of coresets, there have been studies on coreset construction in the distributed/horizontal setting~\cite{balcan2013distributed,phillips2016coresets,bachem2018scalable,lu2020robust}.
To our knowledge, we are the first to consider coreset construction in VFL.

\section{Problem Formulation/Model}
\label{sec:problem}

In this section, we formally define our problems: coresets for vertical regularized linear regression and coresets for vertical \kMeans\ clustering (Problem~\ref{problem:VFL_coreset}).

\paragraph{Vertical federated learning model.}
We first introduce the model of vertical federated learning (VFL).
Let $\mX\subset \R^d$ be a dataset of size $n$ that is vertically separated stored in $T$ data parties ($T\geq 2$).
Concretely, we represent each point $\vx_i\in \mX$ by $\vx_i = (\vx_i^{(1)}, \ldots, \vx_i^{(T)})$ where $\vx_i^{(j)}\in \R^{d_j}$ ($j\in [T]$), and each party $j\in [T]$ holds a local dataset $\mX^{(j)} = \left\{\vx_i^{(j)}\right\}_{i\in [n]}$.
Note that $\sum_{j\in [T]} d_j = d$.
Additionally, if there is a label $y_i\in \R$ for each point $\vx_i\in \mX$, we assume the label vector $\vy\in \R^n$ is stored in Party $T$.
The objective of vertical federated learning is to collaboratively solve certain training problems in the central server with a total communication complexity as small as possible. 

Similar to~\citet[Figure 1]{ding2016kmeans}, we only allow the communication between the central server and each of the $T$ parties, and require the central server to hold the final solution.
Note that the central server can also be replaced with any party in practice.
For the communication complexity, we assume that transporting an integer/floating-point costs 1 unit, and consequently, transporting a $d$-dimensional vector costs $d$ communication units.
See Figure~\ref{fig:model} for an illustration.

\begin{figure}
     \centering
     \begin{subfigure}[b]{0.49\textwidth}
         \centering
         \includegraphics[width=\textwidth]{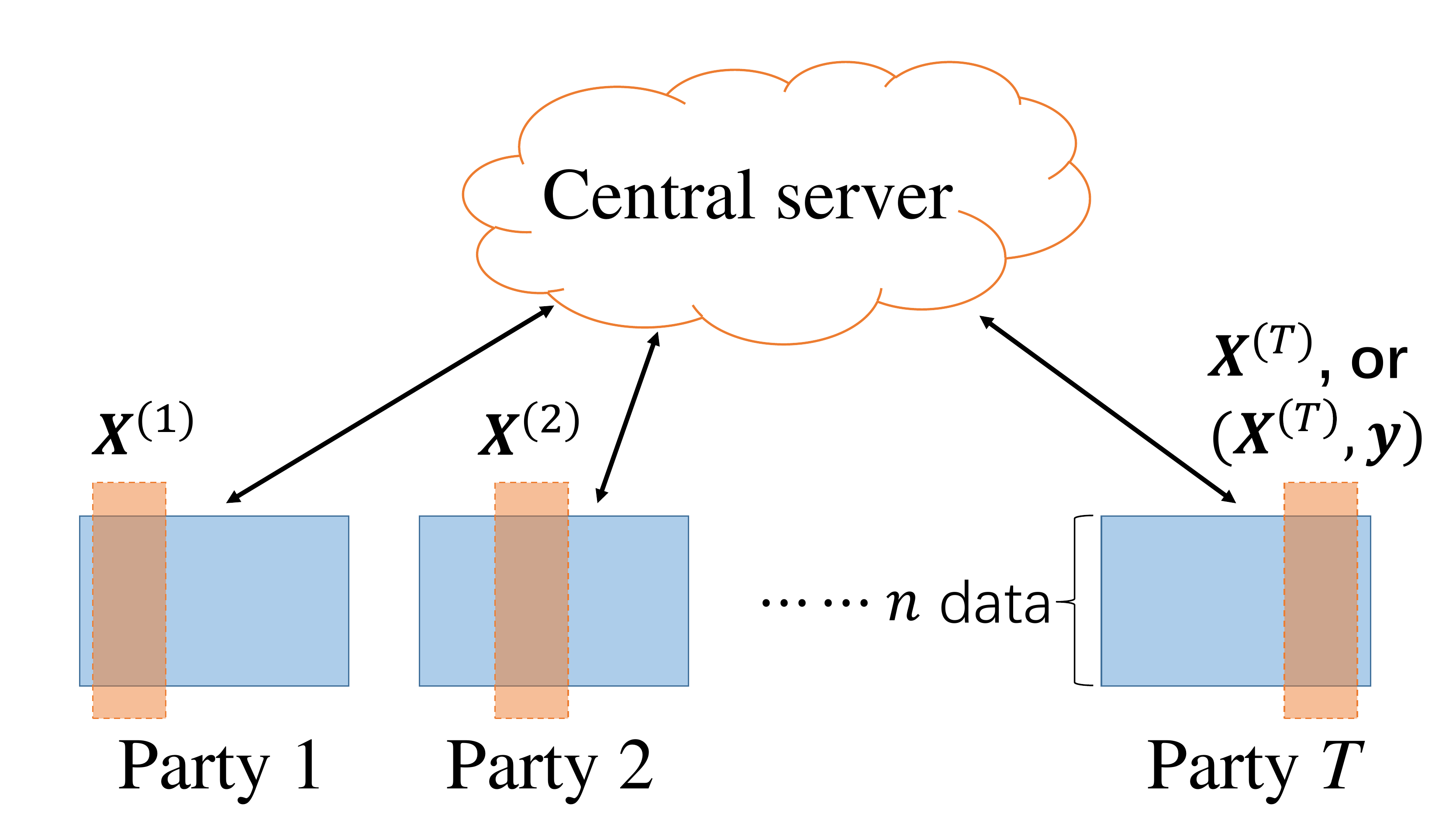}
         \caption{Communication model in VFL}
         \label{fig:model}
     \end{subfigure}
     \hfill
     \begin{subfigure}[b]{0.49\textwidth}
         \centering
         \includegraphics[width=\textwidth]{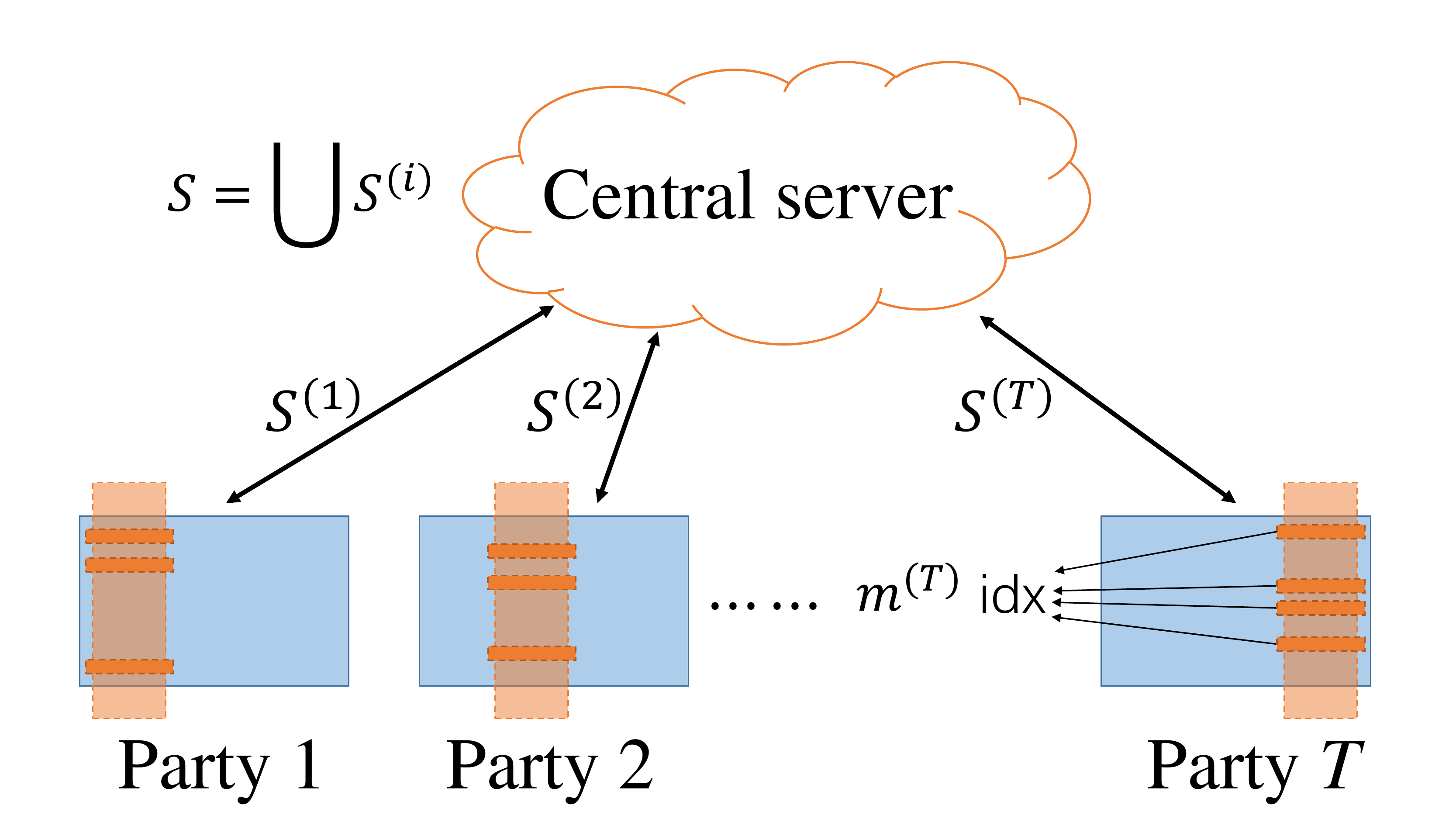}
         \caption{General coreset framework in VFL}
         \label{fig:framework}
     \end{subfigure}
     \caption{Illustration of coreset construction in VFL}
\end{figure}


\paragraph{Vertical regularized linear regression and vertical \kMeans\ clustering.}
%
In this paper, we consider the following two important machine learning problems in the VFL model.

\begin{definition}[\bf{Vertical regularized linear regression (\textrm{VRLR})}]
    \label{def:vertical_regression}
    Given a dataset $\mX\subset \R^d$ together with labels $\vy\in \R^n$ in the VFL model, a regularization function $R: \R^d\rightarrow \R_{\geq 0}$, the goal of the vertical regularized linear regression problem (\textrm{VRLR}) is to compute a vector $\vtheta\in \R^d$ in the server that (approximately) minimizes 
    $
    \costR(\mX, \vtheta) := \sum_{i\in [n]} \costR_i(\mX, \vtheta) = \sum_{i\in [n]} (\vx_i^\top \vtheta - \vy_i)^2 + R(\vtheta),
    $
    and the total communication complexity is as small as possible.
\end{definition}

\noindent
%
%
%

\begin{definition}[\bf{Vertical \kMeans\ clustering (\textrm{VKMC})}]
    \label{def:vertical_kmeans}
    Given a dataset $\mX\subset \R^d$ in the VFL model, an integer $k\geq 1$, let $\gC$ denote the collection of all $k$-center sets $\mC\in \gC$ with $|\mC|=k$ and $d(\cdot, \cdot)$ denote the Euclidean distance.
    The goal of the vertical \kMeans\ clustering problem (\textrm{VKMC}) is to compute a $k$-center set $\mC\in \gC$ in the server that (approximately) minimizes 
    $
    \costC(\mX, \mC) := \sum_{i\in [n]} \costC_i(\mX, \mC) = \sum_{i\in [n]} d(\vx_i,\mC)^2 = \sum_{i\in [n]}\min_{\vc\in \mC} d(\vx_i, \vc)^2,
    $
    and the total communication complexity is as small as possible.
\end{definition}

\noindent
\citet{ding2016kmeans} proposed a similar vertical \kMeans\ clustering problem and provided constant approximation schemes.
They additionally compute an assignment from all points $x_i$ to solution $C$, which requires a communication complexity of at least $\Omega(nT)$.
Due to huge $n$, directly solving \textrm{VRLR} or \textrm{VKMC} is a non-trivial task and may need a large communication complexity.
To this end, we introduce a powerful data-reduction technique, called \emph{coresets}~\citep{harpeled2004on,feldman2011unified,feldman2013turning}.

\paragraph{Coresets for \textrm{VRLR} and \textrm{VKMC}.} 
Roughly speaking, a coreset is a small summary of the original dataset, that approximates the learning objective for every possible choice of learning parameters.
We first define coresets for offline regularized linear regression and \kMeans\ clustering as follows.
As mentioned in Section~\ref{sec:related}, both problems have been well studied in the literature~\cite{feldman2011unified,braverman2016new,chhaya2020coresets,huang2020coresets,cohenaddad2021new,cohenaddad2022towards}.
\begin{definition}[\bf{Coresets for offline regularized linear regression}]
    \label{def:corese_VRLR}
    Given a dataset $\mX\subset \R^d$ together with labels $\vy\in \R^n$ and $\eps\in (0,1)$, a subset $S\subseteq [n]$ together with a weight function $w: S\rightarrow \R_{\geq 0}$ is called an $\eps$-coreset for offline regularized linear regression if for any $\vtheta\in \R^d$,
    \[
    \costR(S,\vtheta) := \sum_{i\in S} w(i)\cdot (\vx_i^\top \vtheta - y_i)^2 + R(\vtheta) \in (1\pm \eps)\cdot \costR(\mX,\vtheta).
    \]
\end{definition}
%
\begin{definition}[\bf{Coresets for offline \kMeans\ clustering}]
    \label{def:corese_VKMC}
    Given a dataset $\mX\subset \R^d$, an integer $k\geq 1$ and $\eps\in (0,1)$, a subset $S\subseteq [n]$ together with a weight function $w: S\rightarrow \R_{\geq 0}$ is called an $\eps$-coreset for offline \kMeans\ clustering if for any $k$-center set $\mC\subset \R^d$,
    \[
    \costC(S,\mC) := \sum_{i\in S} w(i)\cdot d(\vx_i, \mC)^2 \in (1\pm \eps)\cdot \costC(\mX,\mC).
    \]
\end{definition}
%
Now we are ready to give the following main problem.
\begin{problem}[\bf{Coreset construction for \textrm{VRLR} and \textrm{VKMC}}]
    \label{problem:VFL_coreset}
    Given a dataset $\mX\subset \R^d$ (together with labels $\vy\in \R^n$) in the VFL model and $\eps \in (0,1)$, our goal is to construct an $\eps$-coreset for regularized linear regression (or \kMeans\ clustering) in the server, with as small communication complexity as possible. 
    See Figure~\ref{fig:framework} for an illustration.
\end{problem}

\noindent
Note that our coreset is a subset of indices which is slightly different from that in previous work~\citep{harpeled2004on,feldman2011unified,feldman2013turning}, whose coreset consists of weighted points.
This is because we would like to reduce data transportation from parties to the server due to privacy considerations.
Specifically, if the communication schemes for \textrm{VRLR} and \textrm{VKMC} do not need to make data transportation,
then we can avoid data transportation by first applying our coreset construction scheme and then doing the communication schemes based on the coreset. 
Moreover, we have the following theorem that shows how coresets reduce the communication complexity in the VFL models, and the proof is in Section~\ref{sec:coreset_reduce}.

\begin{theorem}[\bf{Coresets reduce the communication complexity for \textrm{VRLR} and \textrm{VKMC}}]
    \label{thm:coreset_reduce} 
    Given $\eps\in (0,1)$, suppose there exist 
    \begin{enumerate}
        \item a communication scheme $A$ that given a (weighted) dataset $\mX\subset \R^d$ together with labels $\vy\in \R^n$ in the VFL model, computes an $\alpha$-approximate solution ($\alpha \geq 1$) for \textrm{VRLR} (or \textrm{VKMC}) in the server with a communication complexity $\Lambda(n)$;
        \item a communication scheme $A'$ that given a (weighted) dataset $\mX\subset \R^d$ together with labels $\vy\in \R^n$ in the VFL model, constructs an $\eps$-coreset for \textrm{VRLR} (or \textrm{VKMC} respecitively) of size $m$ in the server with a communication complexity $\Lambda_0$.
    \end{enumerate}
    Then there exists a communication scheme that given a (weighted) dataset $\mX\subset \R^d$ together with labels $\vy\in \R^n$ in the VFL model, computes an $(1+3\eps)\alpha$-approximate solution ($\alpha \geq 1$) for \textrm{VRLR} (or \textrm{VKMC} respectively) in the server with a communication complexity $\Lambda_0+ 2mT + \Lambda(m)$.
\end{theorem}

\noindent
Usually, $\Lambda(m) = \Omega(mT)$ and $\Lambda_0$ is small or comparable to $Tm$ (see Theorems~\ref{thm:coreset-vrlr} and~\ref{thm:coreset-vkmc} for examples).
Consequently, the total communication complexity by introducing coresets is dominated by $\Lambda(m)$, which is much smaller compared to $\Lambda(n)$.
Hence, coreset can efficiently reduce the communication complexity with a slight sacrifice on the approximate ratio.

\eat{
\begin{table}[!ht]
    \centering
    \begin{tabular}{|c|c|}
    \hline
        $n$ & number of data \\
        $m$ & coreset size \\
        $T$ & number of parties \\
        $\vx_i$ & $i$-th data \\
        $\vy$ & label vector \\
        $\vx_i^{(j)}$ & $i$-th data stored on party $j$ \\
        $\vtheta$ & regression vector \\
        $\mX$ & dataset \\
        $\mX^{(j)}$ & data stored on party $j$ \\
        $\mC$ & center set \\
        $g_i$ & sensitivity upper bound for $\vx_i$ \\
        $g_i^{(j)}$ & sensitivity upper bound for $\vx_i$ on party $j$ \\
        $\gG^{(j)}$ & total sensitivity on party $j$ \\
        $\gG$ & total sensitivity \\
        $\tau$ & separability parameter in clustering \\
    \hline
    \end{tabular}
    \caption{Table of notations}
    \label{tab:table-of-notions}
\end{table}
}

\section{A Unified Scheme for VFL Coresets via Importance Sampling}
\label{sec:unified}

In this section, we propose a unified communication scheme (Algorithm~\ref{alg:dis}) that will be used as a meta-algorithm for solving Problem~\ref{problem:VFL_coreset}.
We assume each party $j\in [T]$ holds a real number $g_i^{(j)}\geq 0$ for data $\vx_i^{(j)}$ in Algorithm~\ref{alg:dis}, that will be computed locally for both \textrm{VRLR} (Algorithm~\ref{alg:coreset-vrlr}) and \textrm{VKMC} (Algorithm~\ref{alg:coreset-vkmc}).
There are three communication rounds in Algorithm~\ref{alg:dis}.
In the first round (Lines 2-4), the server knows all ``local total sensitivities'' $\gG^{(j)}$, takes samples of $[T]$ with probability proportional to $\gG^{(j)}$, and sends $a_j$ to each party $j$, where $a_j$ is the number of local samples of party $j$ for the second round.
In the second round (Lines 5-6), each party samples a collection $S^{(j)}\subseteq [n]$ of size $a_j$ with probability proportional to $g_i^{(j)}$.
The server achieves the union $S = \bigcup_{j\in [T]} S^{(j)}$.
In the third round (Lines 7-8), the goal is to compute weights $w(i)$ for all samples.
In the end, we achieve a weighted subset $(S,w)$.
We propose the following theorem to analyze the performance of Algorithm~\ref{alg:dis} and show that $(S,w)$ is a coreset when size $m$ is large enough.
%

\begin{theorem}[\bf{The performance of Algorithm~\ref{alg:dis}}]
    \label{thm:meta_alg}
    The communication complexity of Algorithm~\ref{alg:dis} is $O(m T)$.
    Let $\eps,\delta\in (0,1/2)$ and $k\geq 1$ be an integer.
    We have
    \begin{itemize}
        \item Let $\zeta = \max\limits_{i\in [n]} \nicefrac{\sup\limits_{\vtheta\in \R^d}\frac{\costR_i(\mX,\vtheta)}{\costR(\mX,\vtheta)}}{\sum_{j\in [T]} g_i^{(j)}}$ and $m = O\left(\eps^{-2}\zeta\gG(d^2\log (\zeta\gG) + \log (1/\delta)) \right)$.
        With probability at least $1-\delta$, $(S,w)$ is an $\eps$-coreset for offline regularized linear regression.
        \item Let $\zeta = \max\limits_{i\in [n]} \nicefrac{\sup\limits_{\mC\in \gC}\frac{\costC_i(\mX,\mC)}{\costC(\mX,\mC)}}{\sum_{j\in [T]} g_i^{(j)}}$ and $m = O\left(\eps^{-2}\zeta\gG(dk\log (\zeta\gG) + \log (1/\delta)) \right)$.
        With probability at least $1-\delta$, $(S,w)$ is an $\eps$-coreset  for offline \kMeans\ clustering.
    \end{itemize}
\end{theorem}

\noindent
The proof can be found in Appendix~\ref{sec:proof_meta}.
The main idea is to show that Algorithm~\ref{alg:dis} simulates a well-known importance sampling framework for offline coreset construction by~\cite{feldman2011unified,braverman2016new}.
The term $\sup_{\vtheta\in \R^d}\frac{\costR_i(\mX,\vtheta)}{\costR(\mX,\vtheta)}$ (or $\sup_{\mC\in \gC}\frac{\costC_i(\mX,\mC)}{\costC(\mX,\mC)}$) is called the \emph{sensitivity} of point $\vx_i$ for \textrm{VRLR} (or \textrm{VKMC}) that represents the maximum contribution of $\vx_i$ over all possible parameters.
Algorithm~\ref{alg:dis} aims to use $\sum_{j\in [T]} g_i^{(j)}$ to estimate the sensitivity of $\vx_i$, and hence, $\zeta$ represents the maximum sensitivity gap over all points.
The performance of Algorithm~\ref{alg:dis} mainly depends on the quality of these estimations $\sum_{j\in [T]} g_i^{(j)}$.
As both $\zeta$ and the total sum $\gG = \sum_{i\in [n], j\in [T]} g_i^{(j)}$ become smaller, the required size $m$ becomes smaller.
Specifically, if both $\zeta$ and $\gG$ only depends on parameters $k,d,T$, the coreset size $m$ is independent of $n$ as expected.
Combining with Theorem~\ref{thm:coreset_reduce}, we can heavily reduce the communication complexity for \textrm{VRLR} or \textrm{VKMC}.

\begin{algorithm}
\caption{A unified importance sampling for coreset construction in the VFL model}
\label{alg:dis}
\begin{algorithmic}[1]
\Require Each party $j\in [T]$ holds data $\vx_i^{(j)}$ together with a real number $g_i^{(j)}\geq 0$, an integer $m\geq 1$
\Ensure a weighted collection $S\subseteq [n]$ of size $|S|\leq m$
\Procedure{DIS}{$m, \{g_i^{(j)}: {i\in [n], j\in [T]}\}$} 
    \State Each party $j\in [T]$ sends $\gG^{(j)} \leftarrow \sum_{i\in [n]} g_i^{(j)}$ to the server. \Comment{1st round begins}
    \State The server computes $\gG = \sum_{j\in [T]} \gG^{(j)}$ and samples a multiset $A\subseteq [T]$ of $m$ samples, where each sample $j\in [T]$ is selected with probability $\nicefrac{\gG^{(j)}}{\gG}$.
    \State The server sends $a_j\leftarrow \# j\in A$ to each party $j\in [T]$. \Comment{1st round ends}
    \State Each party $j\in [T]$ samples a multiset $S^{(j)} \subseteq [n]$ of size $a_j$, where each sample $i\in [n]$ is selected with probability $\nicefrac{g_i^{(j)}}{\gG^{(j)}}$, and sends $S^{(j)}$ to the server.
    \Comment{2nd round begins}
    \State The server broadcasts a multiset $S\leftarrow \bigcup_{j\in [T]}S^{(j)}$ to all parties. \Comment{2nd round ends}
    \State Each party $j\in [T]$ sends $G^{(j)} = \left\{g_i^{(j)} :i\in S\right\}$ to the server. \Comment{3rd round begins } 
    \State The server computes weights $w(i)\leftarrow \nicefrac{\gG}{|S| \cdot \sum_{j\in [T]} g_i^{(j)}}$ for each $i\in S$. \Comment{3rd round ends}
    \State\Return $(S,w)$
\EndProcedure
\end{algorithmic}
\end{algorithm}

\paragraph{Privacy issue.}
We consider the privacy of the proposed scheme from two aspects: coreset construction and model training.
As for the coreset construction part (Algorithm 1), the privacy leakage comes from the "sensitivity score" $g_i^{(j)}$ of the data points in different parties.
To tackle this problem, we can use secure aggregation \cite{bonawitz2017practical}  to transport the sum $g_i=\sum_{j=1}^T g_i^{(j)}$ to the server without revealing the exact values of $g_i^{(j)}$s (Line 7 of Algorithm~\ref{alg:dis}).
The server only knows $(S, w)$ and $\gG^{(j)}$s.
For the model training part, we can apply the secure VFL algorithms if existed, e.g., using homomorphic encryption on SAGA for regression (it is an extension from SGD to SAGA \cite{hardy2017private}).

Note that the VFL communication model in Section~\ref{sec:problem} is assumed to be semi-honest.
Suppose some party $j$ is malicious, then it can report a large enough $\gG^{(j)}$ (Line 2 of Algorithm~\ref{alg:dis}) such that the server sets the number of samples $a_j\approx m$ in party $j$ (Line 4 of Algorithm~\ref{alg:dis}).
Consequently, party $j$ can sample a large multi-set $S^{(j)}$ which heavily affects the resulting set $S$.
For instance, by reporting $S^{(j)}$ of uniform samples, party $j$ can make $S$ close to uniform sampling and loss the theoretical guarantees in Theorem~\ref{thm:meta_alg}.


\section{Coreset Construction for \textrm{VRLR}}\label{sec:vrlr}
In this section, we discuss the coreset construction for \textrm{VRLR}. We first show that it is generally hard to construct a strong coreset for \textrm{VRLR}. Then, we show how to communication-efficiently construct coresets for \textrm{VRLR} under mild assumption.
All missing proofs can be found in Section~\ref{sec:proof_vrlr}.

With slightly abuse of notation, we denote $\mX\in\R^{n\times d} , \mX^{(j)}\in\R^{n\times d_j}$ to be the data matrix of whole data and the data matrix stored on party $j$ respectively. Since there are labels $\vy$ stored on party $T$, $\mX^{(T)}$ has dimension $n\times (d_T+1)$.

\paragraph{Communication complexity lower bound for \textrm{VRLR}} We first show that it is \textit{hard} to compute the coreset for \textrm{VRLR} by proving an $\Omega(n)$ deterministic communication complexity lower bound.

\begin{theorem}[\bf{Communication complexity of coreset construction for \textrm{VRLR}}]\label{thm:hardness-vrlr}
    Let $T\geq 2$. 
    Given constant $\eps \in (0,1)$, any deterministic communication scheme that constructs an $\eps$-coreset for \textrm{VRLR} requires a communication complexity $\Omega(n)$.  
\end{theorem}

The communication complexity lower bound for linear regression has also been considered in the HFL setting~\citep{vempala2020communication}, e.g., \citet{vempala2020communication} also gets a deterministic communication complexity lower bound. 
Theorem \ref{thm:hardness-vrlr} shows that linear regression in the VFL setting is ``hard'' and thus we may need to add data assumptions to get theoretical guarantees for coreset construction.

\paragraph{Communication-efficient coreset construction for \textrm{VRLR}} Now we show that under mild condition, we can construct a strong coreset for \textrm{VRLR} using $o(n)$ number of communication. Specifically, we assume the data $\mX$ satisfies the following assumption, which will be justified in the appendix.

\begin{restatable}{assumption}{assvrlr}\label{ass:vrlr}
    Let $\mU^{(j)} \in \R^{n\times d'_j}$ denote the orthonormal basis of the column space of $\mX^{(j)}$ stored on party $j$ ($\mU^{(T)}$ denotes the orthonormal basis of $[\mX^{(T)}, y]$), and then the matrix $\mU = [\mU^{(1)}, \mU^{(2)},\dots,\mU^{(T)}]$ has smallest singular value $\sigma_{\min}(\mU) \ge \gamma >0$.
\end{restatable}


%
Intuitively, $\gamma\in (0,1]$ represents the degree of orthonormal among data in different parties. As the larger $\gamma$ is, the more orthonormal among the column spaces of $X^{(j)}$, and thus $U$ is more close to the orthonormal basis computed on $X$ directly.
Now we introduce our coreset construction algorithm for \textrm{VRLR} (Algorithm \ref{alg:coreset-vrlr}). At a very high level perspective, we let each party $j$ to compute a coreset $S^{(j)}$ based on its own data $\mX^{(j)}$, and combine all the $S^{(j)}$ together to obtain a final coreset $S$. More specifically, for each party $j$, we let it to compute $\mU^{(j)} = [\vu_1^{(j)},\dots,\vu_n^{(j)}]^\top$ based on the data $\mX^{(j)}$, and set $g_i^{(j)} = \|\vu^{(j)}_i\|^2+\frac{1}{n}$ to be the weight of data $i$ on party $j$. Then, we set $g_i = \sum_{j\in [T]}g_i^{(j)}$ to be the final weight of data $\vx_i$ and want to sample $m$ samples using weight $g_i$. To do this, we apply the DIS procedure (Algorithm \ref{alg:dis}).

\begin{algorithm}[!t]
\caption{Vertical federated coreset construction for Regularized Linear Regression (VRLR)}
\label{alg:coreset-vrlr}
\begin{algorithmic}[1]
\Require Each party $j\in [T]$ holds the data $\vx_i^{(j)}$ for all $i\in [n]$, coreset size $m$.
\For{each party $j \in [T]$}
    \State Compute orthornormal basis $\mU^{(j)} = [\vu_1^{(j)},\dots,\vu_n^{(j)}]^\top$ of $\mX^{(j)}$
    \State $g_i^{(j)} \leftarrow \|\vu_i^{(j)}\|^2+\frac{1}{n}$ for all $i\in [n]$
\EndFor
\State \Return $(S,w)\leftarrow\texttt{DIS}(m, \{g_i^{(j)}\})$
\end{algorithmic}
\end{algorithm}

\begin{theorem}[\bf{Coresets for \textrm{VRLR}}]\label{thm:coreset-vrlr}
    For a given dataset $\mX \subset \R^d$ satisfying Assumption~\ref{ass:vrlr}, number of parties $T \geq 1$ and constants $\eps, \delta \in (0,1)$, with probability at least $1-\delta$, Algorithm \ref{alg:coreset-vrlr} constructs an $\eps$-coreset for \textrm{VRLR} of size $m = O(\eps^{-2}\gamma^{-2}d(d^2\log{\gamma^{-2}d}+\log{1/\delta}))$, and uses communication complexity $O(mT)$.
\end{theorem}

Note that the coreset size and the total communication are all independent on $n$, and thus when combined with Theorem \ref{thm:coreset_reduce}, using coreset construction can reduce the communication complexity for \textrm{VRLR}.
When Assumption \ref{ass:vrlr} is not satisfied, Algorithm \ref{thm:coreset-vrlr} is not guaranteed to return a strong coreset.
However, as shown in the following remark, it will return another kind of coreset called \emph{robust coreset} ~\cite{feldman2011unified,huang2018epsilon,wang2021robust}, which allows a small portion of data to be treated as outliers and excluded both in $S$ and $\mX$ when evaluating the quality of $S$.
The outliers represent a small percentage of data with unbounded sensitivity gap.
More details can be found in the Theorem~\ref{thm:robust-coreset-vrlr}.
\begin{remark}[Robust coreset for \textrm{VRLR}]
Given a dataset $\mX\subset \R^d$ together with labels $\vy\in \R^n$, $\eps\in (0,1)$ and $\beta\in [0,1)$, a subset $S\subseteq [n]$ together with a weight function $w: S\rightarrow \R_{\geq 0}$ is called a $(\beta,\eps)$-robust coreset for offline regularized linear regression if for any $\vtheta\in \R^d$, there exists a subset $O_{\vtheta} \subseteq [n]$ such that $|O_{\vtheta}|/n \leq \beta$, $|S \cap O_{\vtheta}|/|S| \leq \beta$ and
    \[
    \costR(S \backslash O_{\vtheta},\vtheta) \in \costR(\mX \backslash O_{\vtheta},\vtheta) \pm \eps \cdot \costR(\mX,\vtheta).
    \]
If Assumption \ref{ass:vrlr} is not satisfied, for $m=O((\eps\beta T)^{-2}d^6)$, 
Algorithm \ref{alg:coreset-vrlr} will return a $(\beta, \eps)$-robust coreset for \textrm{VRLR} with communication complexity $O(mT)$.
\end{remark}

\section{Coreset Construction for \textrm{VKMC}}
\label{sec:vkmc}
In this section, we discuss the coreset construction for \textrm{VKMC}.
Similar to \textrm{VRLR}, we first show it generally requires $\Omega(n)$ communication complexity to construct a coreset for \textrm{VKMC}, and then we show that it is possible to vastly reduce the communication complexity (Algorithm~\ref{alg:coreset-vkmc}) under mild data assumption.
All missing proofs can be found in Section~\ref{sec:proof_vkmc}.
\paragraph{Communication complexity lower bound for \textrm{VKMC}.}
We first present an $\Omega(n)$ communication complexity lower bound for constructing an $\eps$-coreset for \textrm{VKMC} in the following theorem.
\begin{theorem}[\bf{Communication complexity of coreset construction for \textrm{VKMC}}]
    \label{thm:hardness-vkmc}
    Let $d\geq T \geq 2$. Given a constant $\eps \in (0,1)$ and an integer $k \geq 3$, any randomized communication scheme that constructs an $\eps$-coreset for \textrm{VKMC} with probability 0.99 requires a communication complexity $\Omega(n)$.
\end{theorem}

\noindent
Different from \textrm{VRLR}, we have a randomized communication complexity lower bound for \textrm{VKMC}. 
Similarly, we also need to introduce certain data assumptions to get theoretical guarantees for coreset construction due to this hardness result.
\paragraph{Communication-efficient coreset construction for \textrm{VKMC}} Now we show how to communication-efficiently construct coresets for \textrm{VKMC} under mild condition.
Specifically, we assume that the data satisfies the following assumption, which will be justified in the appendix.
\begin{restatable}{assumption}{assvkmc}\label{ass:vkmc}
    There exists $\tau \geq 1$ and some party $t \in [T]$ such that $\norm{\vx_{i}-\vx_{j}}^{2} \leq \tau \norm{\vx^{(t)}_i-\vx^{(t)}_j}^2$ for any $i, j \in [n]$.
\end{restatable}

This assumption says that, there is a party that is ``important'', and any two data points which can be differentiated can also be differentiated on that party to some extent. Specifically, as $\tau$ is more close to 1, Assumption 5.1 implies that there exists a party $t\in [T]$ whose local pairwise distances $\|x_i^{(t)} - x_j^{(t)}\|$s are close to the corresponding global pairwise distances $\|x_i - x_j\|$s.
Then we introduce our coreset construction algorithm for \textrm{VKMC} (Algorithm~\ref{alg:coreset-vkmc}).
For the input, note that there exist several constant approximation algorithms for \kMeans\  \citep{kumar2004simple, vassilvitskii2006k}.
The widely used $k$-means++ algorithm \citep{vassilvitskii2006k} provides an $O(\ln{k})$-approximation and performs well in practice.
Similar to Algorithm~\ref{alg:coreset-vrlr} for \textrm{VRLR}, Algorithm~\ref{alg:coreset-vkmc} also applies Algorithm~\ref{alg:dis} after computing $g_i^{(j)}$ locally.
The key is to construct local sensitivities $g_i^{(j)}$ to upper bound both $\zeta$ and $\gG$ in Theorem~\ref{thm:meta_alg}.
The derivation of the local sensitivities $g_i^{(j)}$ defined in Line 10 is partly inspired by \cite{varadarajan2012sensitivity}, which upper bounds the total sensitivity of a point set in clustering problems by projecting points onto an optimal solution.
Intuitively, if some party $t$ satisfies Assumption~\ref{ass:vkmc}, a constant factor approximate solution computed locally in party $t$ can also induce a global one.
Then by projecting points onto this global constant approximation, we can prove that $g_i^{(t)}$ (scaled by some constant factor) is an upper bound of the global sensitivity of $\vx_i$ for any $i\in[n]$.
Though unaware of which party satisfies Assumption~\ref{ass:vkmc}, it suffices to sum up $g_i^{(j)}$ over $j\in[T]$, only costing an additional $T$ in $\gG$.
Finally, we can upper bound $\zeta$ by $O(\tau)$ and $\gG$ by $O(kT)$ respectively.
The main theorem is as follows.
%


\begin{algorithm}[!t]
\caption{Vertical federated coreset construction for \kMeans\ Clustering (VKMC)}
\label{alg:coreset-vkmc}
\begin{algorithmic}[1]
\Require Each party $j\in [T]$ holds the data $\vx_i^{(j)}$ for all $i\in [n]$, coreset size $m$, number of centers $k$, an $\alpha$-approximation algorithm $\gA$ (e.g. $k$-means++).
\Ensure a weighted collection $S\subseteq [n]$ of size $|S|\leq m$
\ForAll{party $j \in [T]$}
    \State $\mC^{(j)} \leftarrow \gA(\{\vx_i^{(j)}\}_{i\in [n]})$. Note that $\mC^{(j)} = \{\vc_1^{(j)}, \vc_2^{(j)},\dots, \vc_{k}^{(j)}\}$.
    \State Initialize $\mB_l^{(j)}=\varnothing$ for $l\in[k]$.
    \ForAll{$i \in [n]$}
        \State $\pi(i) \leftarrow \argmin_{l\in[k]} d(\vx_i^{(j)},\vc_{l}^{(j)})$\Comment{a mapping to find the closest center locally.}
        \State $\mB_{\pi(i)}^{(j)} \leftarrow \mB_{\pi(i)}^{(j)} \cup i $.
    \EndFor
    \State $\text{cost}^{(j)} \leftarrow \sum_{i\in [n]} d(\vx_i^{(j)}, \mC^{(j)})^2$ \Comment{$d(\vx_i^{(j)},\mC^{(j)})=d(\vx_i^{(j)},\vc_{\pi(i)}^{(j)})$}
    \ForAll{$i\in[n]$}
    \State $l \leftarrow \pi(i)$, $g_i^{(j)} \leftarrow \frac{\alpha d(\vx_i^{(j)}, \vc_l^{(j)})^2}{\text{cost}^{(j)}} + \frac{\alpha \sum_{i'\in \mB_l^{(j)}}d(\vx_{i'}^{(j)}, \vc_l^{(j)})^2}{|\mB_l^{(j)}| \text{cost}^{(j)}} + \frac{2\alpha}{|\mB_l^{(j)}|}$.
    \EndFor
\EndFor
\State \Return $(S,w)\leftarrow\texttt{DIS}(m, \{g_i^{(j)}\})$
\end{algorithmic}
\end{algorithm}

\begin{theorem}[\bf{Coresets for \textrm{VKMC}}]
    \label{thm:coreset-vkmc}
    For a given dataset $\mX \subset \R^d$ satisfying Assumption~\ref{ass:vkmc}, an $\alpha$-approximation algorithm for \kMeans\ with $\alpha=O(1)$, integers $k \geq 1$, $T \geq 1$ and constants $\eps, \delta \in (0,1)$, with probability at least $1-\delta$, Algorithm \ref{alg:coreset-vkmc} constructs an $\eps$-coreset for \textrm{VKMC} of size $m = O(\eps^{-2}\alpha\tau kT(dk\log{(\alpha\tau kT)}+\log{1/\delta}))$, and uses communication complexity $O(mT)$.
\end{theorem}

\noindent
Again, note that both the coreset size and communication complexity are independent of $n$.
Thus, using Algorithm \ref{alg:coreset-vkmc} together with other baseline algorithms can drastically reduce the communication complexity.
Similar to \textrm{VRLR}, we have the following remark when the data assumption (Assumption \ref{ass:vkmc}) is not satisfied. More details can be found in the Theorem \ref{thm:robust-coreset-vkmc}.
\begin{remark}[Robust coreset for \textrm{VKMC}]
Given a dataset $\mX\subset \R^d$, an integer $k\geq 1$, $\eps\in (0,1)$ and $\beta\in [0,1)$, a subset $S\subseteq [n]$ together with a weight function $w: S\rightarrow \R_{\geq 0}$ is called a $(\beta,\eps)$-robust coreset for offline \kMeans\ clustering if for any $\mC\subset \R^d$, there exists a subset $O_{\mC} \subseteq [n]$ such that $|O_{\mC}|/n \leq \beta$, $|S \cap O_{\mC}|/|S| \leq \beta$ and
    \[
    \costC(S \backslash O_{\mC},\mC) \in \costC(\mX \backslash O_{\mC},\mC) \pm \eps \cdot \costC(\mX,\mC).
    \]
If Assumption \ref{ass:vkmc} is not satisfied, for $m=O((\eps\beta)^{-2}k^5d)$ Algorithm \ref{alg:coreset-vkmc} will return a $(\beta, \eps)$-robust coreset for \textrm{VKMC} with communication complexity $O(mT)$.
\end{remark}

\section{Numerical Experiments}\label{sec:experiments}

In this section, we present the numerical experiments, which corroborate our theoretical results.
We conduct experiments on a single system that simulates the distributed settings.\footnote{The codes are available at \url{https://github.com/haoyuzhao123/coreset-vfl-codes}.}

\paragraph{Empirical setup.}
We conduct experiments on the \dataset{YearPredictionMSD} dataset~\citep{bertin2011million} for both \textrm{VRLR} and \textrm{VKMC}.
\dataset{YearPredictionMSD} dataset has 515345 data, and each data contains 90 features and a corresponding label. 
We assume there are $T=3$ parties and each party stories 30 distinct features. 
For \textrm{VRLR}, we split the data into a training set with size 463715 and a testing set with size 51630. 
We consider ridge regression in \textrm{VRLR} by letting $R(\vtheta) = \lambda \norm{\vtheta}^2$ for $\lambda = 0.1n$ where $n$ is the dataset size. 
For \textrm{VKMC}, there is only one training set with size 515345 and without labels. 
We choose $k=10$ (10 centers) and we normalize each feature with mean 0 and standard deviation 1 for \textrm{VKMC}.

For \textrm{VRLR}, we consider two baselines: 1) \textsc{Central} as the procedure that transfers all data to the central server and solves the problem using scikit-learn package~\cite{scikit-learn}; 2) \textsc{SAGA} as using \citep{defazio2014saga}'s algorithm to optimize in a VFL fashion.
For \textrm{VKMC}, we also consider two baselines: 1) \textsc{Kmeans++} as the procedure that transfers all data to the central server and clusters using \textsc{Kmeans++}~\citep{vassilvitskii2006k}; 2) \textsc{DistDim} by \citep{ding2016kmeans}.

For each baseline, we compare our coreset algorithm with uniform sampling. 
We use \textsc{C-X} to denote coreset sampling followed by algorithm \textsc{X} and \textsc{U-X} for uniform sampling followed by algorithm \textsc{X}, e.g. \textsc{C-DistDim} means that we apply coreset construction and then use \textsc{DistDim} algorithm. 
We compare \textsc{C-X} and \textsc{U-X} with different sizes, and each experiment is repeated 20 times.

\paragraph{Empirical results.} Figure \ref{fig:ridge-res} shows our results for \textrm{VRLR} and Figure \ref{fig:kmeans-res} shows our results for \textrm{VKMC}. Table \ref{tab:res-yearprediction} summarize the results. For \textrm{VRLR}, since it is a supervised learning problem, we report the testing loss; for \textrm{VKMC}, it is an unsupervised learning task and the cost refers to the training loss on the full training data.

\begin{figure}
    \centering
    \begin{subfigure}[b]{0.525\textwidth}
        \centering
        \includegraphics[width=\textwidth]{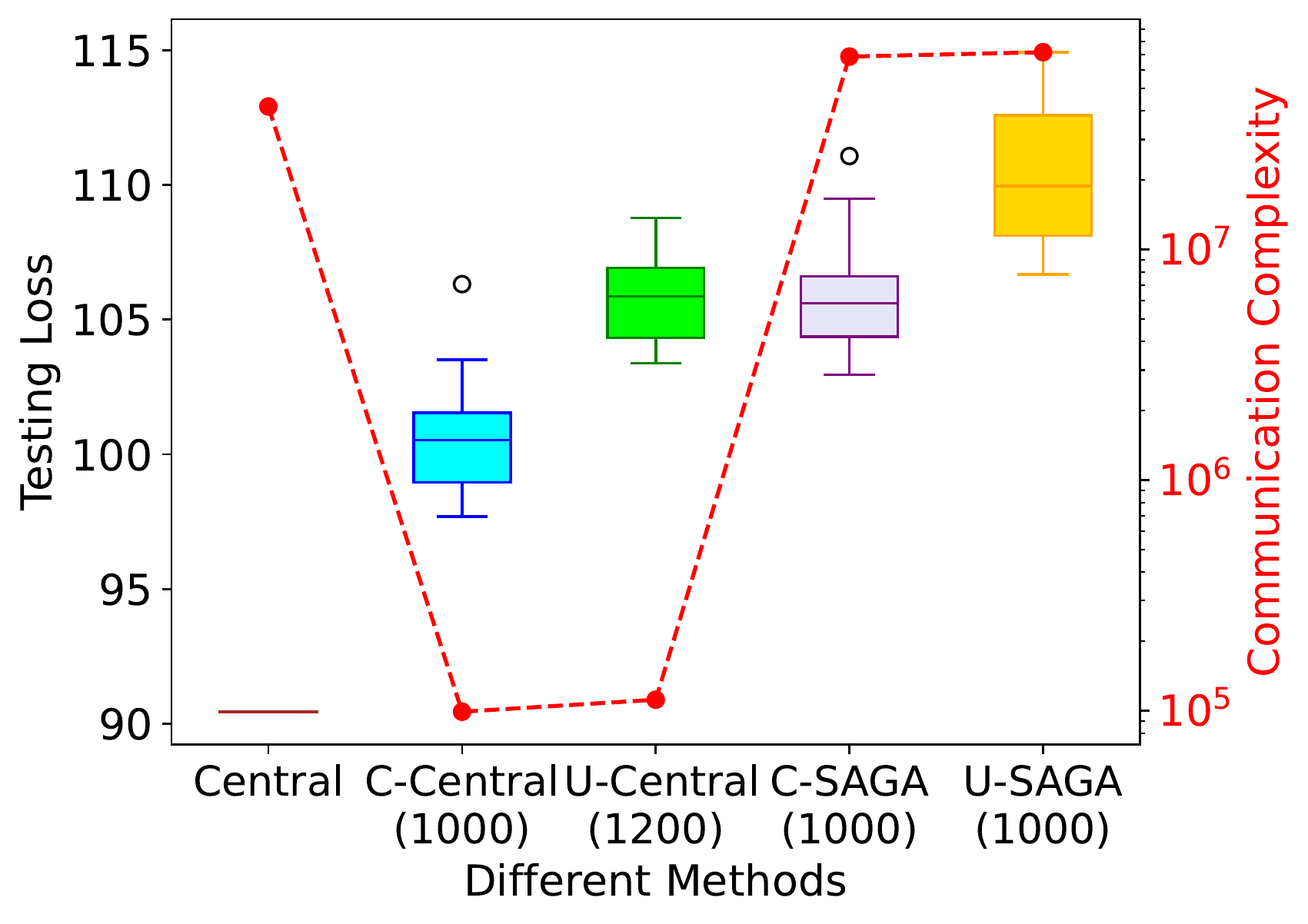}.
    \end{subfigure}
    \hfill
    \begin{subfigure}[b]{0.465\textwidth}
        \centering
        \includegraphics[width=\textwidth]{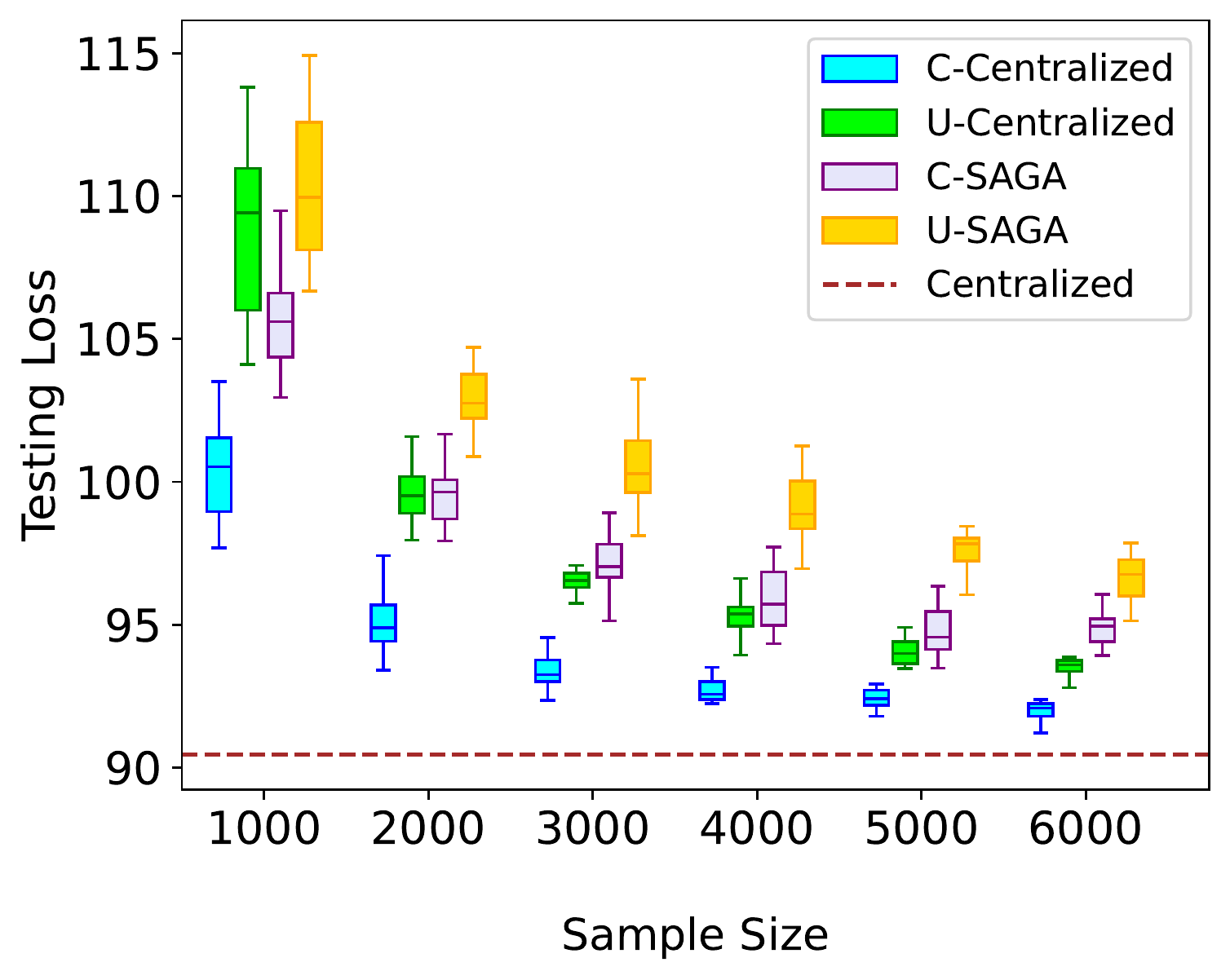}
    \end{subfigure}
    \caption{Left: Testing loss and communication complexity of \textrm{VRLR} for different methods. C and U means using coreset or uniform sampling. The number in the parentheses denotes the sample size. Right: Testing loss of \textrm{VRLR} for different methods under multiple sample sizes.}
    \label{fig:ridge-res}
\end{figure}

\begin{figure}
    \centering
    \begin{subfigure}[b]{0.525\textwidth}
        \centering
        \includegraphics[width=\textwidth]{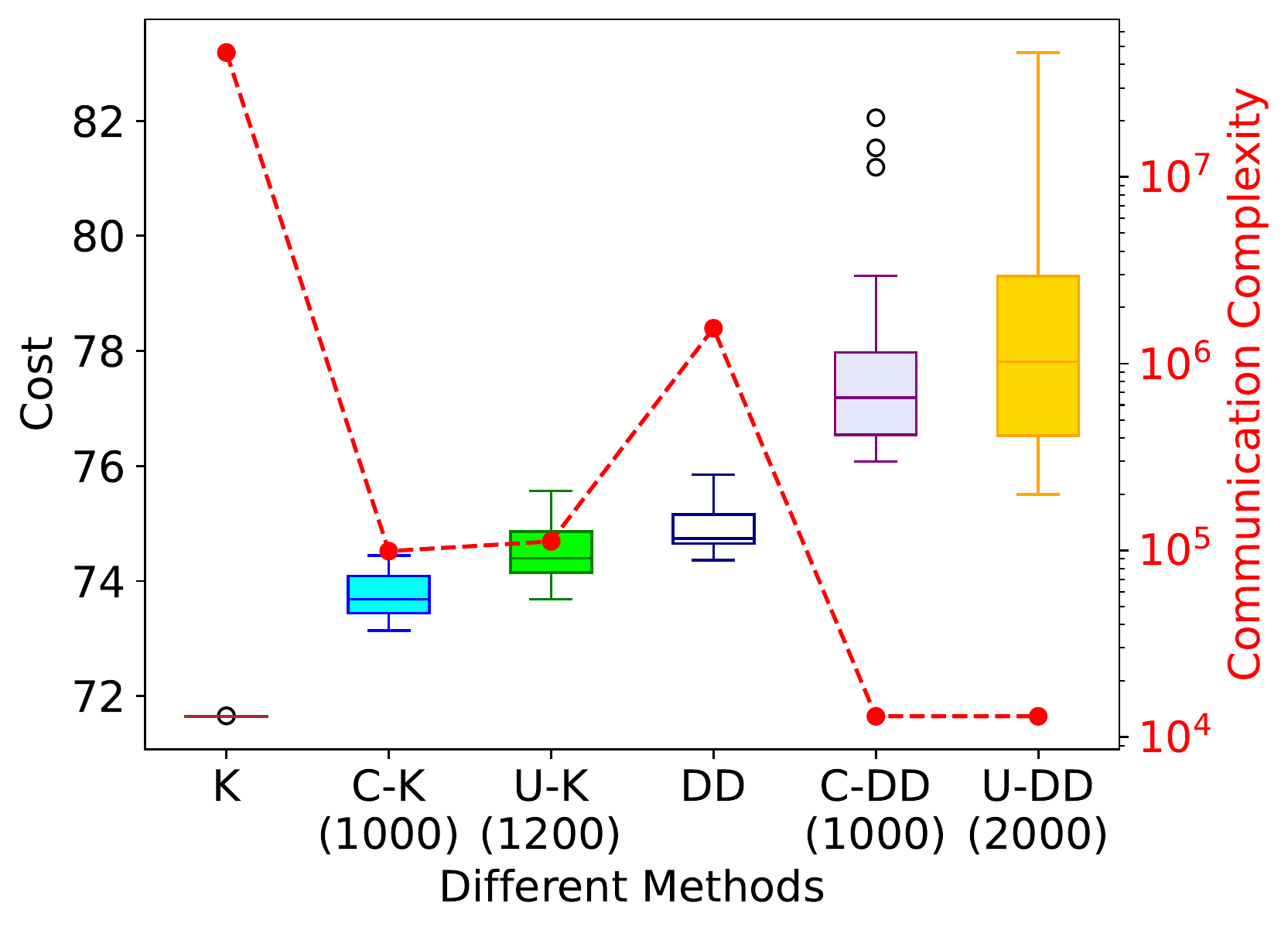}
    \end{subfigure}
    \hfill
    \begin{subfigure}[b]{0.465\textwidth}
        \centering
        \includegraphics[width=\textwidth]{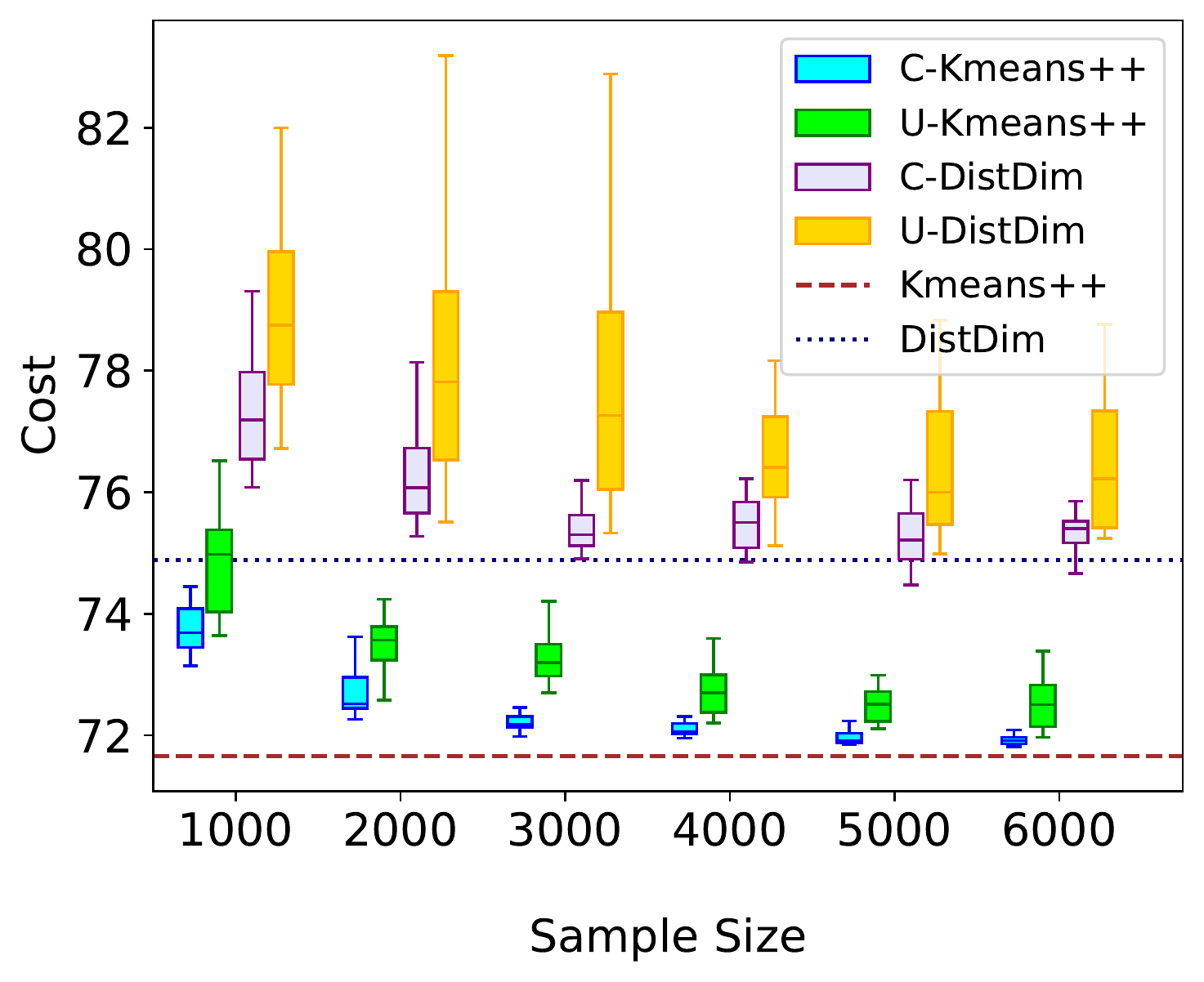}
    \end{subfigure}
    \caption{Left: Cost and communication complexity of \textrm{VKMC} for different methods. C and U means using coreset sampling or uniform sampling. The number in the parentheses denote the sample size. Right: Cost of \textrm{VKMC} for different methods under multiple sample sizes.}
    \label{fig:kmeans-res}
\end{figure}

\begin{table}[h]
\caption{Results of \textrm{VRLR} and \textrm{VKMC} on \dataset{YearPredictionMSD} dataset. Left: results for \textrm{VRLR}. Right: results for \textrm{VKMC}. The average and std. are computed using the 20 repeated experiments. The communication complexity denotes the average communication complexity, and the number in the parenthesis denotes the fraction of coreset construction (or uniform sampling respectively).}
\label{tab:res-yearprediction}
\vspace{1mm}
\begin{subtable}{.51\linewidth}
\footnotesize
 \centering
 \resizebox{\linewidth}{!}{%
 \begin{tabular}{ccccccc}
 \hline
 \multicolumn{2}{c}{\makecell{Alg \\ (size)}} & \multicolumn{1}{c}{\makecell{Cost \\ avg/std}} & \multicolumn{1}{c}{\makecell{Com. \\ compl.}} & Alg & \makecell{Cost \\ avg/std} & \makecell{Com. \\ compl.}\\
 \hline
 \multicolumn{2}{c}{\textsc{Central}} & 90.45/0.00 & 4.2e7  & & &\\
 \cline{1-7}
 1000 & \parbox[t]{2mm}{\multirow{6}{*}{\rotatebox[origin=c]{90}{\textsc{C-Central}}}} & 100.50/2.11 & 9.9e4(0.09) & \parbox[t]{2mm}{\multirow{6}{*}{\rotatebox[origin=c]{90}{\textsc{U-Central}}}} & 108.79/2.79 & 9.3e4(0.03) \\
 2000 & & 95.01/0.95 & 2.0e5(0.09) & & 99.65/1.01 & 1.9e5(0.03) \\
 3000 & & 93.39/0.63 & 3.0e5(0.09) & & 96.68/0.73 & 2.8e5(0.03) \\
 4000 & & 92.73/0.41 & 4.0e5(0.09) & & 95.32/0.65 & 3.7e5(0.03) \\
 5000 & & 92.42/0.33 & 5.0e5(0.09) & & 94.08/0.48 & 4.7e5(0.03) \\
 6000 & & 91.97/0.32 & 5.9e5(0.09) & & 93.54/0.44 & 5.6e5(0.03) \\
 \cline{1-7}
 \multicolumn{2}{c}{\textsc{SAGA}} & N/A & N/A & & &\\
 \cline{1-7}
 1000 & \parbox[t]{2mm}{\multirow{6}{*}{\rotatebox[origin=c]{90}{\textsc{C-SAGA}}}} & 105.93/2.17 & 6.9e7(<0.01) & \parbox[t]{2mm}{\multirow{6}{*}{\rotatebox[origin=c]{90}{\textsc{U-SAGA}}}} & 110.43/2.43 & 7.2e7(<0.01) \\
 2000 & & 99.55/0.96 & 1.4e8(<0.01) & & 102.96/1.09 & 1.6e8(<0.01) \\
 3000 & & 97.13/0.98 & 1.9e8(<0.01) & & 100.64/1.44 & 2.4e8(<0.01) \\
 4000 & & 95.90/1.04 & 2.5e8(<0.01) & & 99.10/1.20 & 3.3e8(<0.01) \\
 5000 & & 94.75/0.85 & 3.0e8(<0.01) & & 97.64/0.59 & 3.9e8(<0.01) \\
 6000 & & 94.83/0.65 & 3.6e8(<0.01) & & 96.88/1.12 & 4.6e8(<0.01) \\
 \hline
 \end{tabular}
 }
\end{subtable}
\begin{subtable}{.49\linewidth}
\footnotesize
 \centering
 \resizebox{\linewidth}{!}{%
 \begin{tabular}{ccccccc}
 \hline
 \multicolumn{2}{c}{\makecell{Alg \\ (size)}} & \multicolumn{1}{c}{\makecell{Cost \\ avg/std}} & \multicolumn{1}{c}{\makecell{Com. \\ compl.}} & Alg & \makecell{Cost \\ avg/std} & \makecell{Com. \\ compl.}\\
 \hline
 \multicolumn{2}{c}{\textsc{Kmeans++}} & 71.65/0.00 & 4.6e7  & & &\\
 \cline{1-7}
 1000 & \parbox[t]{2mm}{\multirow{6}{*}{\rotatebox[origin=c]{90}{\textsc{C-Kmeans++}}}} & 73.76/0.38 & 9.9e4(0.09) & \parbox[t]{2mm}{\multirow{6}{*}{\rotatebox[origin=c]{90}{\textsc{U-Kmeans++}}}} & 74.91/0.81 & 9.3e4(0.03) \\
 2000 & & 72.68/0.35 & 2.0e5(0.09) & & 73.52/0.44 & 1.9e5(0.03) \\
 3000 & & 72.23/0.17 & 3.0e5(0.09) & & 73.25/0.39 & 2.8e5(0.03) \\
 4000 & & 72.14/0.19 & 4.0e5(0.09) & & 72.74/0.39 & 3.7e5(0.03) \\
 5000 & & 71.97/0.13 & 5.0e5(0.09) & & 72.54/0.37 & 4.7e5(0.03) \\
 6000 & & 71.92/0.09 & 5.9e5(0.09) & & 72.57/0.44 & 5.6e5(0.03) \\
 \cline{1-7}
 \multicolumn{2}{c}{\textsc{DistDim}} & 74.89/0.00 & 1.5e6 & & &\\
 \cline{1-7}
 1000 & \parbox[t]{2mm}{\multirow{6}{*}{\rotatebox[origin=c]{90}{\textsc{C-DistDim}}}} & 77.75/1.78 & 1.3e4(0.70) & \parbox[t]{2mm}{\multirow{6}{*}{\rotatebox[origin=c]{90}{\textsc{U-DistDim}}}} & 78.87/1.44 & 7.0e3(0.43) \\
 2000 & & 76.82/0.85 & 2.5e4(0.72) & & 78.13/2.10 & 1.3e4(0.47) \\
 3000 & & 75.52/0.64 & 3.7e4(0.73) & & 77.85/2.23 & 1.9e4(0.48) \\
 4000 & & 75.49/0.45 & 4.9e4(0.74) & & 76.70/1.27 & 2.5e4(0.48) \\
 5000 & & 75.27/0.50 & 6.1e4(0.74) & & 76.38/1.17 & 3.1e4(0.49) \\
 6000 & & 75.32/0.33 & 7.3e4(0.74) & & 76.44/1.03 & 3.7e4(0.49) \\
 \hline
 \end{tabular}
 }
\end{subtable}
\end{table}

\paragraph{Coreset sampling performs close to the baseline with less communication.} From the results, we find that using our coreset can achieve a similar loss compared to the baseline, while the communication complexity is reduced drastically. 
Specifically by Table~\ref{tab:res-yearprediction}, our coreset algorithm \textsc{C-Central} can use less than 0.4\% of training data (2000/463715) and achieve a $95.01/90.45\approx 1.05$-approximate solution for \textrm{VRLR} compared to the baseline \textsc{Central}.
Observe that a larger coreset size leads to a smaller cost and a larger communication complexity. 
From Figures \ref{fig:ridge-res} and \ref{fig:kmeans-res} (left), using coresets can reduce 50-100x communication complexity compared with the original baselines.

\paragraph{Coreset performs better than uniform sampling under the same communication.} 
From Figures~\ref{fig:ridge-res} and \ref{fig:kmeans-res} (right), we observe that our coresets always achieve a better solution than uniform sampling under the same sample size. 
Table \ref{tab:res-yearprediction} also reflects this trend.
Under the same sample size, the communication complexity by uniform sampling is slightly lower than that of coreset, since there is no need to transfer weights in uniform sampling. 
Thus, we also compare the performance of our coresets and uniform sampling under the same communication complexity.
%
%
%
From Figures~\ref{fig:ridge-res} and \ref{fig:kmeans-res} (left), we find that for different baselines, our coreset algorithms still achieve better testing loss/training cost while using fewer or the same communication, compared to uniform sampling. 
%
%


\paragraph{Coreset and uniform sampling may also make the problem feasible.} It is also interesting to observe that \textsc{SAGA} will not converge (or very slowly) on the original \textrm{VRLR} problem (Table~\ref{tab:res-yearprediction}), possibly because of the large dataset and the ill-conditioned optimization problem. However, by applying the coreset/uniform sampling, \textsc{SAGA} works for \textrm{VRLR}. This also indicates the effectiveness of our framework and the importance to reduce the dependency on $n$ (the dataset size).

\section{Conclusion and Future Directions}
\label{sec:conclusion}

In this paper, we first consider coreset construction in the vertical federated learning setting.
We propose a unified coreset framework for communication-efficient VFL, and apply the framework to two important learning tasks: regularized linear regression and $k$-means clustering.
We verify the efficiency of our coreset algorithms both theoretically and empirically, which can drastically alleviate the communication complexity while still maintaining the solution quality.

Our work initializes the topic of introducing coresets to VFL, which leaves several future directions.
Firstly, our VFL coreset size is still larger than that of offline coresets for both \textrm{VRLR} and \textrm{VKMC}, even under certain data assumptions.
One direction is to further improve the coreset size.
Another interesting direction is to extend coreset construction to other learning tasks in the VFL setting, e.g., logistic regression or gradient boosting trees.

\bibliography{ref}
\bibliographystyle{abbrvnat}

\newpage
\section*{Checklist}

\begin{enumerate}

\item For all authors...
\begin{enumerate}
  \item Do the main claims made in the abstract and introduction accurately reflect the paper's contributions and scope?
    \answerYes{}
  \item Did you describe the limitations of your work?
    \answerYes{See Section~\ref{sec:conclusion}.}
  \item Did you discuss any potential negative societal impacts of your work?
    \answerNo{}
  \item Have you read the ethics review guidelines and ensured that your paper conforms to them?
    \answerYes{}
\end{enumerate}

\item If you are including theoretical results...
\begin{enumerate}
  \item Did you state the full set of assumptions of all theoretical results?
    \answerYes{See Assumption~\ref{ass:vrlr} and Assumption~\ref{ass:vkmc}, and we also provide justification of these data assumtions in Section~\ref{sec:justify}.}
        \item Did you include complete proofs of all theoretical results?
    \answerYes{All missing proofs can be found in the appendix.}
\end{enumerate}

\item If you ran experiments...
\begin{enumerate}
  \item Did you include the code, data, and instructions needed to reproduce the main experimental results (either in the supplemental material or as a URL)?
    \answerYes{See section \ref{sec:experiments} footnote.}
  \item Did you specify all the training details (e.g., data splits, hyperparameters, how they were chosen)?
    \answerYes{}
        \item Did you report error bars (e.g., with respect to the random seed after running experiments multiple times)?
    \answerYes{}
        \item Did you include the total amount of compute and the type of resources used (e.g., type of GPUs, internal cluster, or cloud provider)?
    \answerNo{All experiments are done using a single computer}
\end{enumerate}

\item If you are using existing assets (e.g., code, data, models) or curating/releasing new assets...
\begin{enumerate}
  \item If your work uses existing assets, did you cite the creators?
    \answerYes{See the baseline algorithms mentioned in Section~\ref{sec:experiments}.}
  \item Did you mention the license of the assets?
    \answerNA{}
  \item Did you include any new assets either in the supplemental material or as a URL?
    \answerYes{We provided the code for our experiments.}
  \item Did you discuss whether and how consent was obtained from people whose data you're using/curating?
    \answerNA{}
  \item Did you discuss whether the data you are using/curating contains personally identifiable information or offensive content?
    \answerNA{}
\end{enumerate}

\item If you used crowdsourcing or conducted research with human subjects...
\begin{enumerate}
  \item Did you include the full text of instructions given to participants and screenshots, if applicable?
    \answerNA{}
  \item Did you describe any potential participant risks, with links to Institutional Review Board (IRB) approvals, if applicable?
    \answerNA{}
  \item Did you include the estimated hourly wage paid to participants and the total amount spent on participant compensation?
    \answerNA{}
\end{enumerate}

\end{enumerate}


\newpage
\appendix
\section*{Appendix}

\section{Additional Experiments}

In this section, we present some additional experiments. This section is organized as follow: in Section \ref{sec:exp-parties}, we conduct experiments using different number of parties (as opposed to three parties in Section \ref{sec:experiments}); in Section \ref{sec:exp-reg}, we test our methods using other regularizer for \textrm{VRLR}, e.g., Lasso; in Section \ref{sec:exp-centers}, we test our methods in \textrm{VKMC} with different number of centers; and finally in Section \ref{sec:exp-dataset}, we conduct experiments on another dataset (\dataset{KC House} Dataset~\citep{kagglekc}).

\subsection{Different number of parties}\label{sec:exp-parties}

In this section, we test our algorithms using different number of parties. We choose to use five parties ($T=5$) in this section instead of three parties in Section \ref{sec:experiments}.

\paragraph{Empirical setup} Most of the experimental setups are the same as those in Section \ref{sec:experiments}, except that now we use 5 parties instead of 3 parties. There are 90 dimensions for a single data in \dataset{YearPredictionMSD} dataset, and we let each party hold 18 dimensions. Besides, changing the number of parties does not affect the performance of \algname{U-Central} and \algname{U-SAGA} (but the number of communication will change due to different number of parties), and we reuse the results from Section \ref{sec:experiments} and recalculate the number of communications.

\paragraph{Empirical results} Figure \ref{fig:ridge-res-5p} and \ref{fig:kmeans-res-5p} summarize our results for \textrm{VRLR} and \textrm{VKMC} respectively. Note that all the observations in Section \ref{sec:experiments} hold for 5 parties.

\begin{figure}[!h]
    \centering
    \begin{subfigure}[b]{0.525\textwidth}
        \centering
        \includegraphics[width=\textwidth]{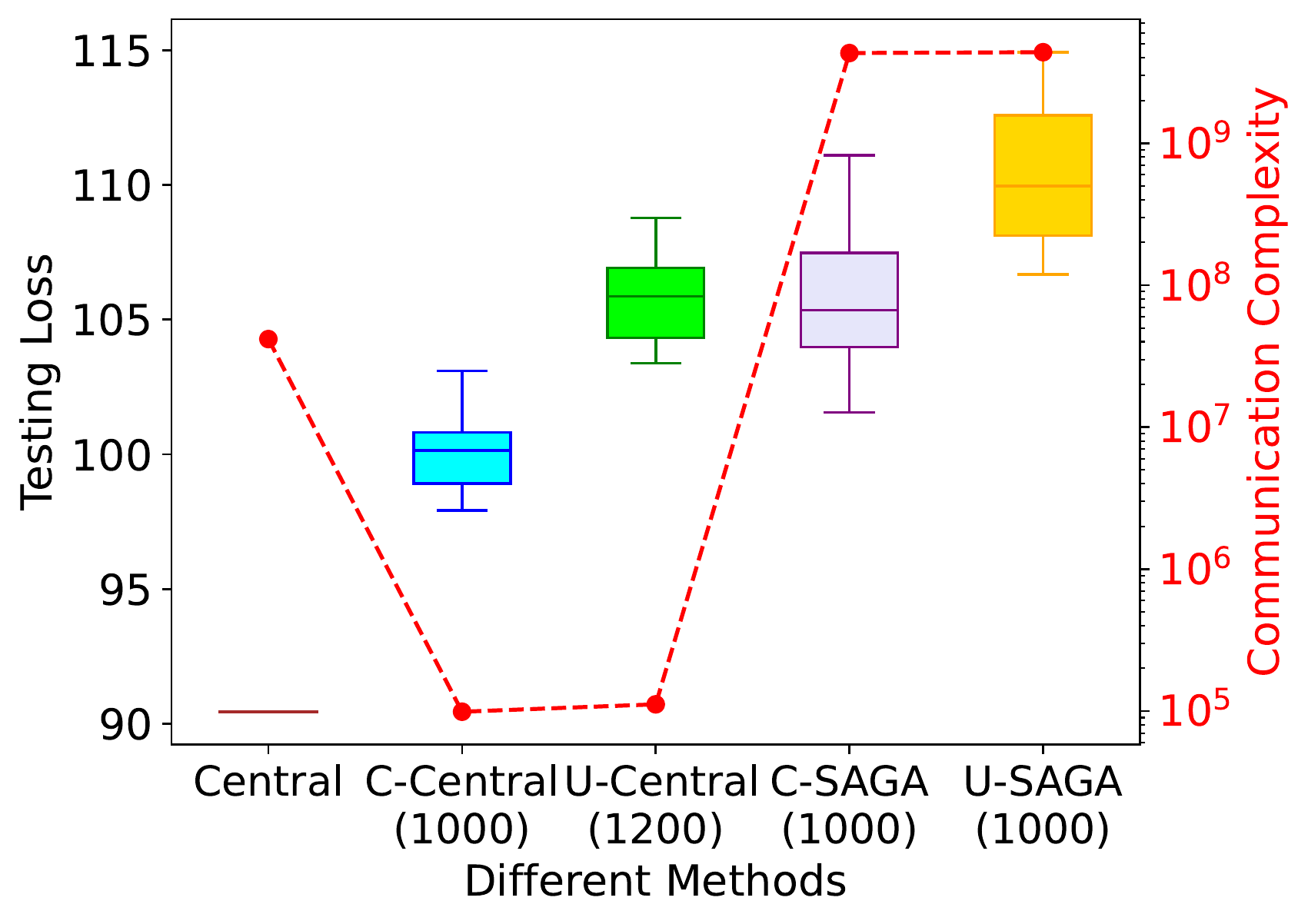}.
    \end{subfigure}
    \hfill
    \begin{subfigure}[b]{0.465\textwidth}
        \centering
        \includegraphics[width=\textwidth]{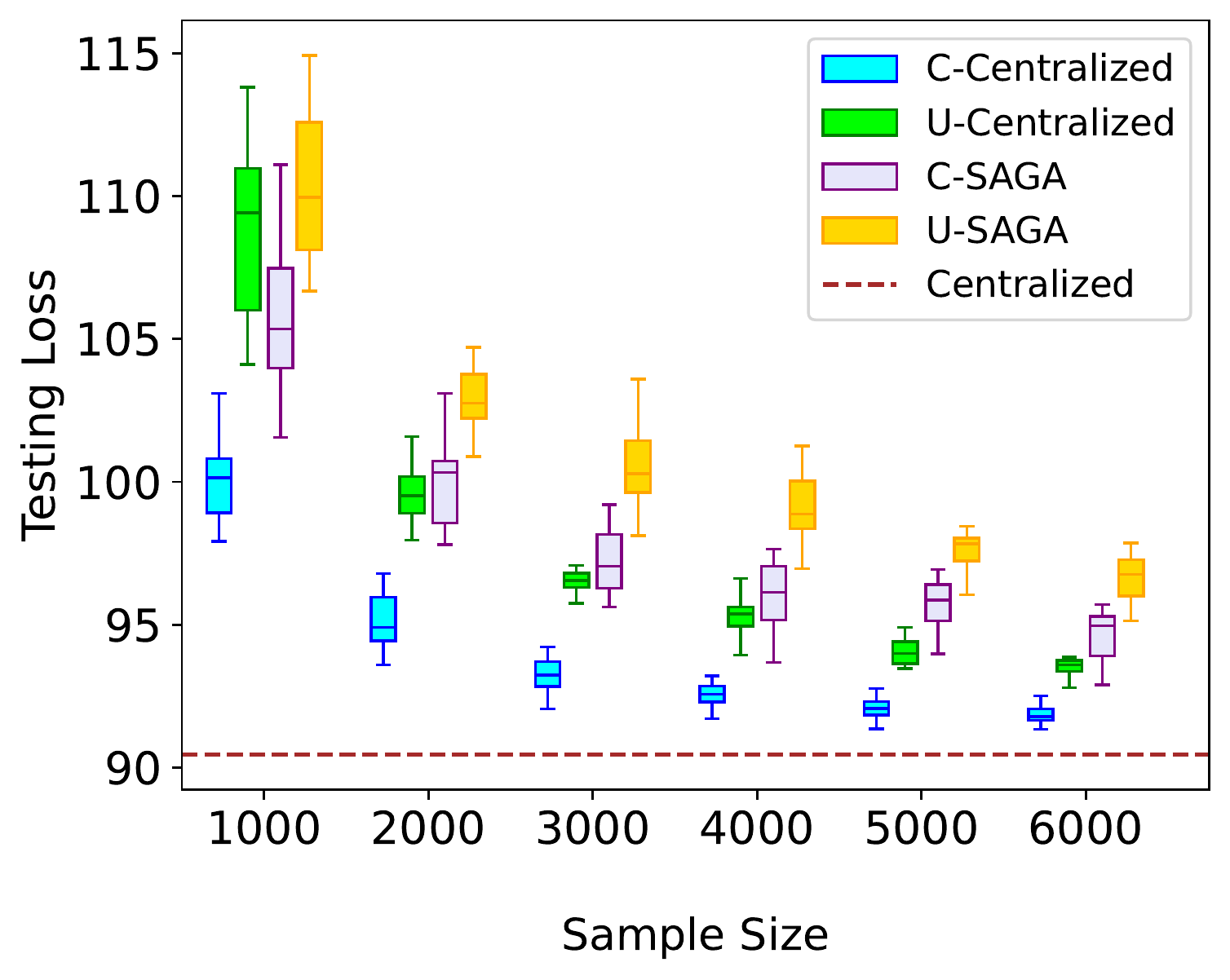}
    \end{subfigure}
    \caption{\emph{Results for 5 parties (Section \ref{sec:exp-parties})} Left: Testing loss and communication complexity of \textrm{VRLR} for different methods. C and U means using coreset or uniform sampling. The number in the parentheses denote the sample size. Right: Testing loss of \textrm{VRLR} for different methods under multiple sample sizes.}
    \label{fig:ridge-res-5p}
\end{figure}

\begin{figure}
    \centering
    \begin{subfigure}[b]{0.525\textwidth}
        \centering
        \includegraphics[width=\textwidth]{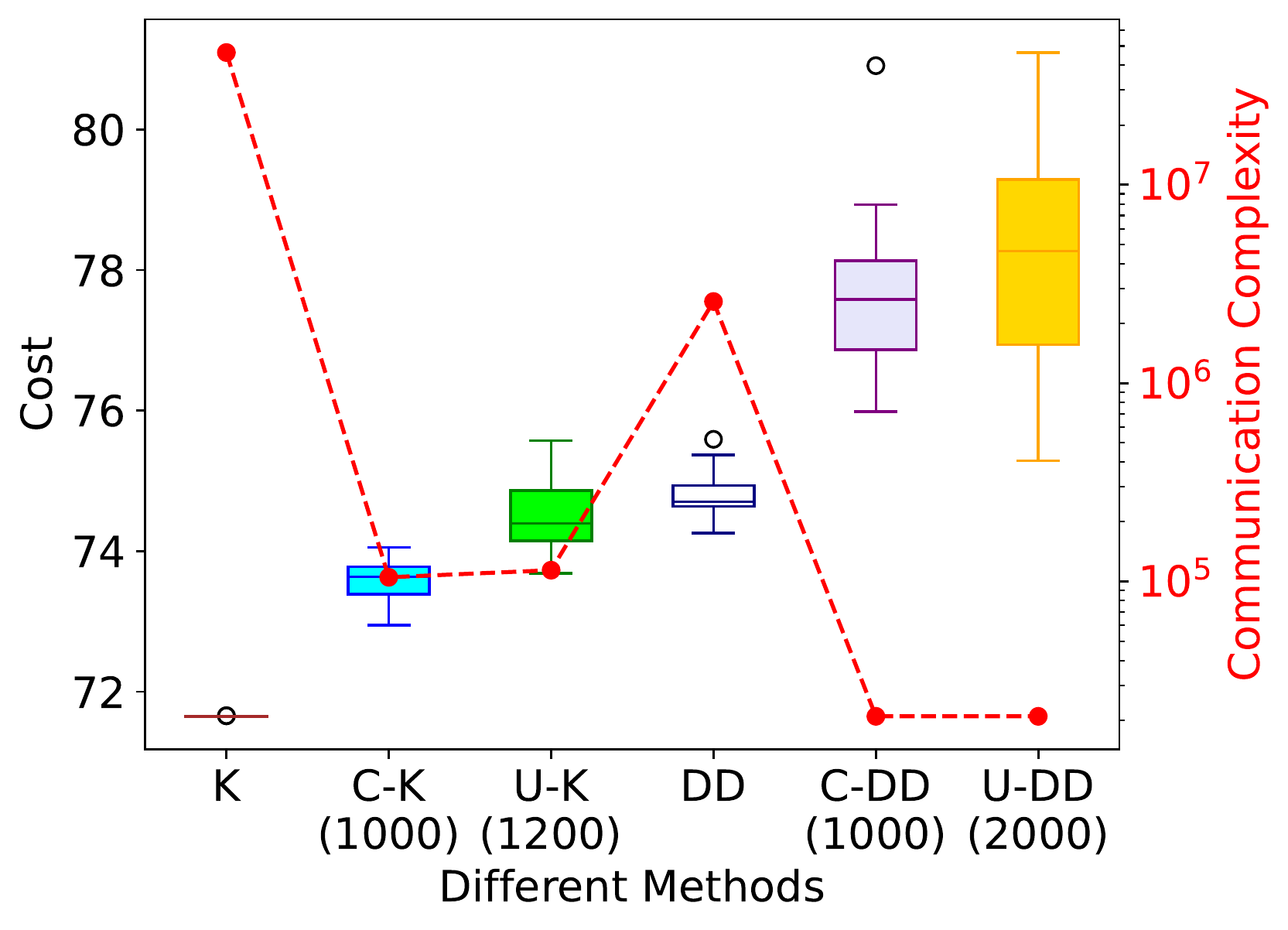}
    \end{subfigure}
    \hfill
    \begin{subfigure}[b]{0.465\textwidth}
        \centering
        \includegraphics[width=\textwidth]{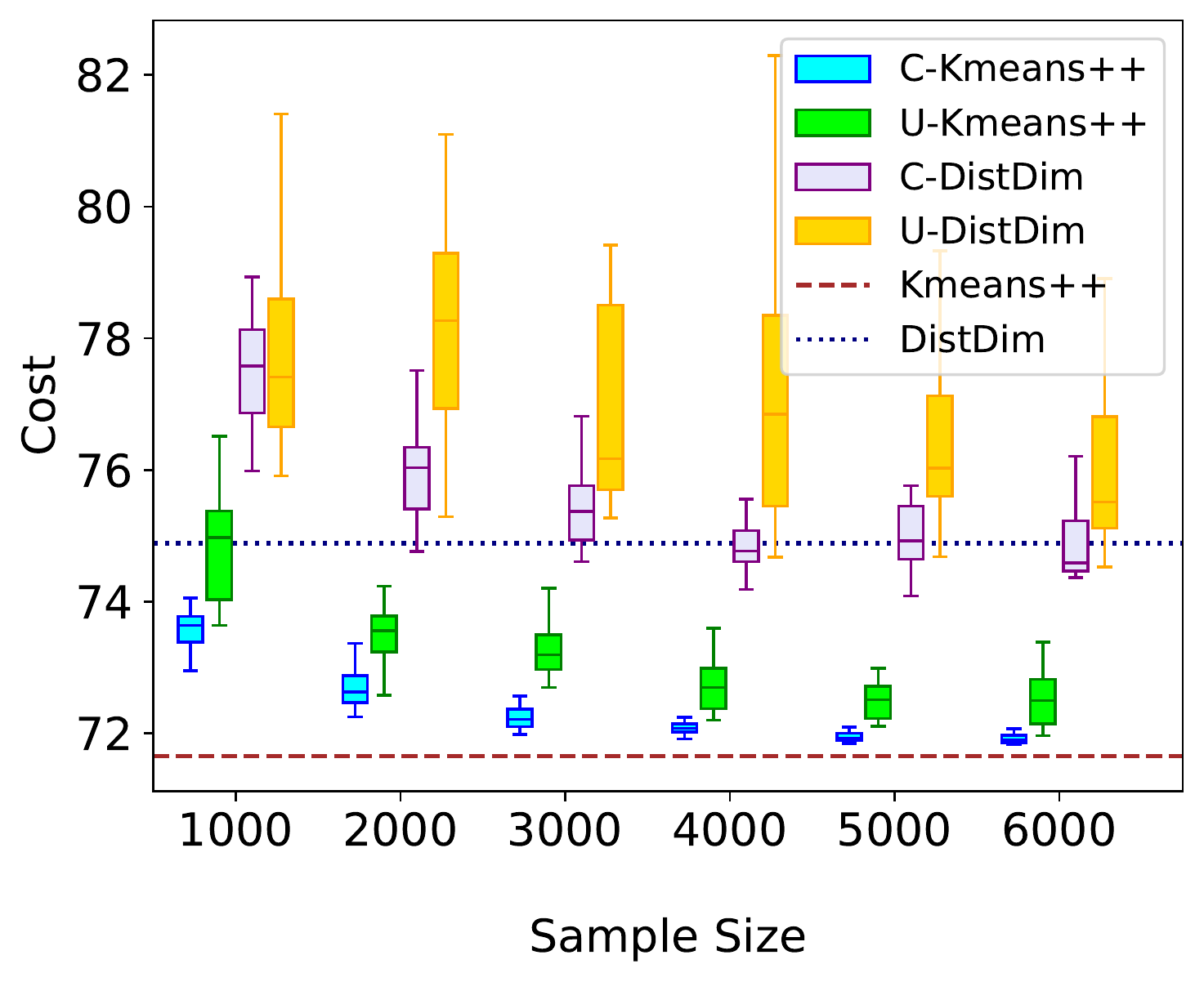}
    \end{subfigure}
    \caption{\emph{Results for 5 parties (Section \ref{sec:exp-parties})} Left: Cost and communication complexity of \textrm{VKMC} for different methods. C and U means using coreset sampling or uniform sampling. The number in the parentheses denote the sample size. Right: Cost of \textrm{VKMC} for different methods under multiple sample sizes.}
    \label{fig:kmeans-res-5p}
\end{figure}

\subsection{Different regularizer for \textrm{VRLR}}\label{sec:exp-reg}
In this part, we consider using different regularizers in \textrm{VRLR}.

\paragraph{Empirical setup} We consider three different regression problems: plain linear regression, Lasso regression, and elastic nets. In Section \ref{sec:experiments}, we consider the Ridge regression ($R(\vtheta) = 0.1n \norm{\vtheta}_2^2$ where $n$ is the dataset size), and in this part, linear regression denotes the optimization problem where $R(\vtheta) = 0$, Lasso regression denotes the problem where $R(\vtheta) = 2n \norm{\vtheta}_1$, and elastic net denotes the problem where $R(\vtheta) = 2n \norm{\vtheta}_1 + n\norm{\vtheta}_2^2$. All the experiments setup remains the same as Section \ref{sec:experiments}, except the for Lasso regression and elastic nets, there is no \algname{SAGA} solver and we only compare \algname{C-Central} and \algname{U-Central} with \algname{Central}.

\paragraph{Empirical results} We plot the training loss instead of the testing loss since we are comparing different objective functions. Figure \ref{fig:lr-res}, \ref{fig:lasso-res}, and \ref{fig:en-res} show the empirical results in this part. Note that all the observations in Section \ref{sec:experiments} also hold: (1) coreset sampling and uniform sampling can drastically reduce the communication complexity where nearly maintain the solution performance, and (2) coreset performs better than uniform sampling under the same number of communication.

\begin{figure}
    \centering
    \begin{subfigure}[b]{0.525\textwidth}
        \centering
        \includegraphics[width=\textwidth]{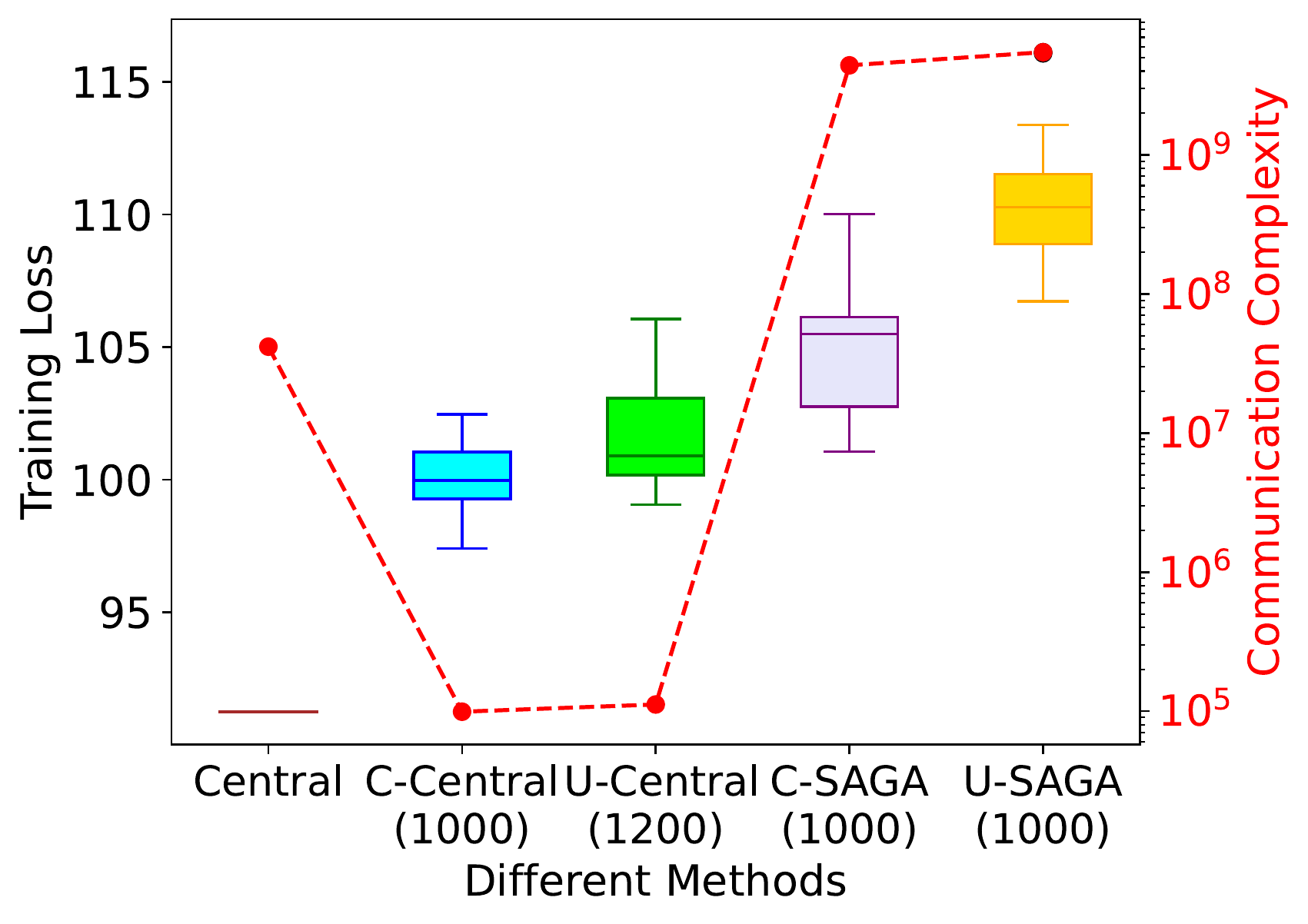}.
    \end{subfigure}
    \hfill
    \begin{subfigure}[b]{0.465\textwidth}
        \centering
        \includegraphics[width=\textwidth]{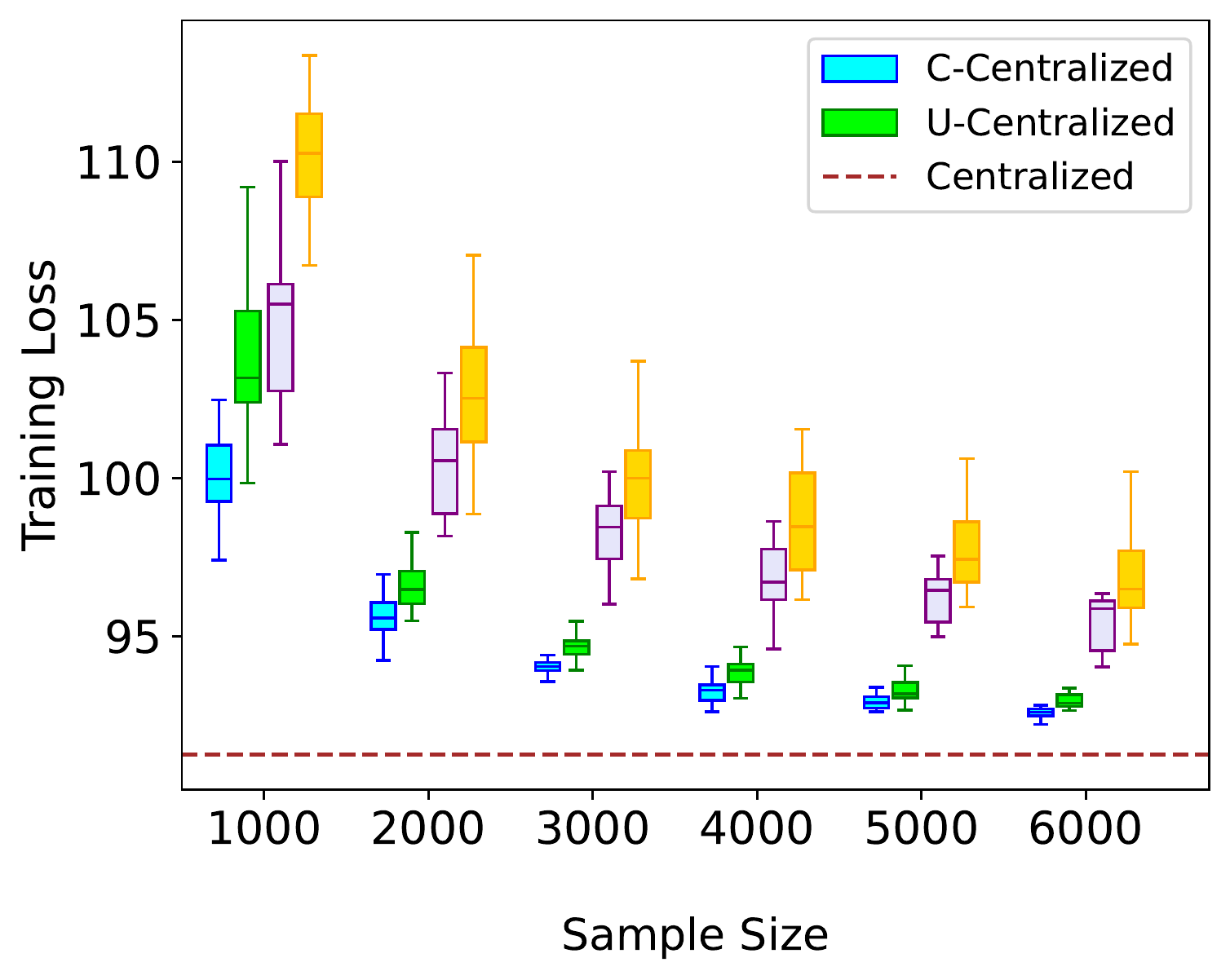}
    \end{subfigure}
    \caption{\emph{Results for linear regression (Section \ref{sec:exp-reg})} Left: Training loss and communication complexity of \textrm{VRLR} for different methods. C and U means using coreset or uniform sampling. The number in the parentheses denote the sample size. Right: Testing loss of \textrm{VRLR} for different methods under multiple sample sizes.}
    \label{fig:lr-res}
\end{figure}

\begin{figure}
    \centering
    \begin{subfigure}[b]{0.525\textwidth}
        \centering
        \includegraphics[width=\textwidth]{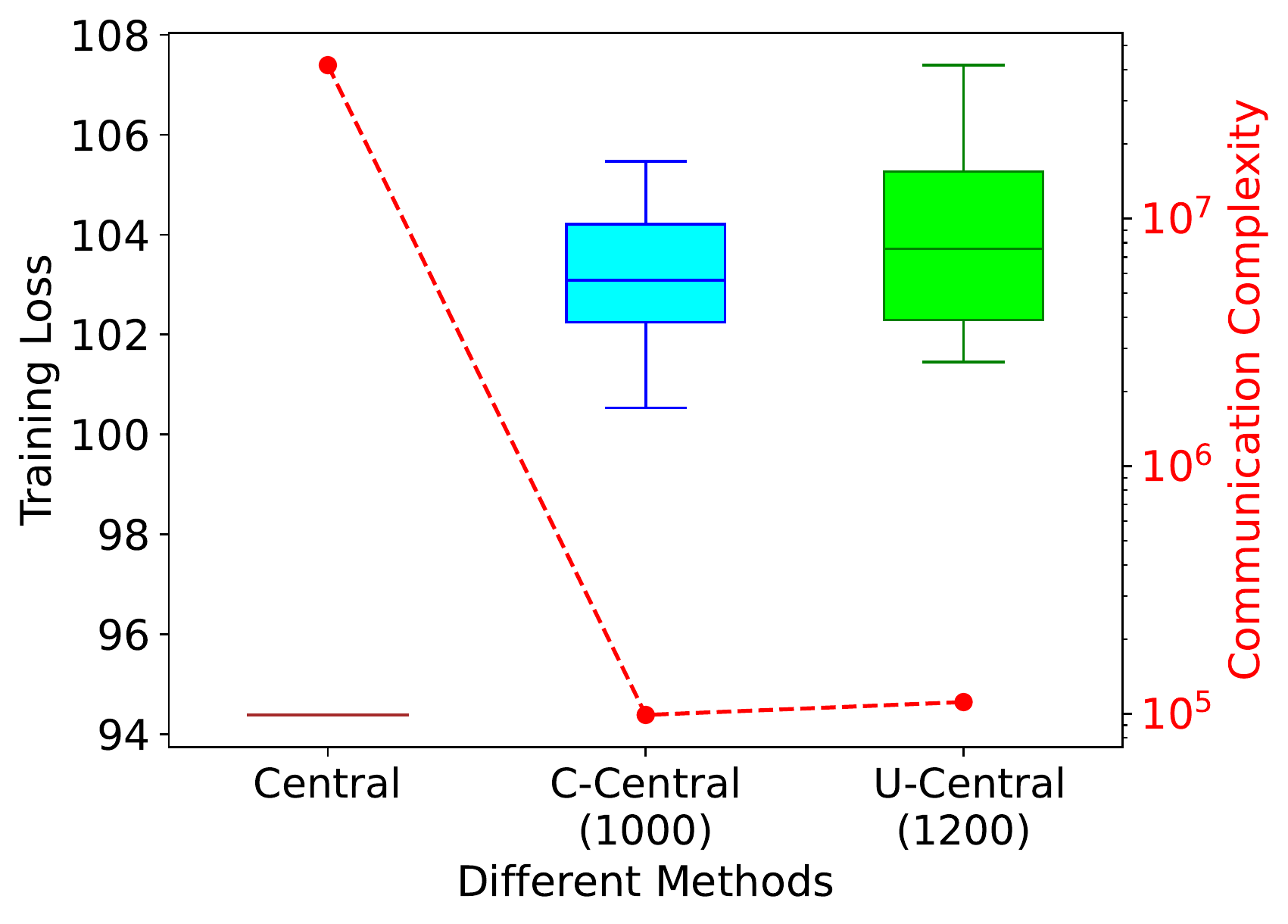}.
    \end{subfigure}
    \hfill
    \begin{subfigure}[b]{0.465\textwidth}
        \centering
        \includegraphics[width=\textwidth]{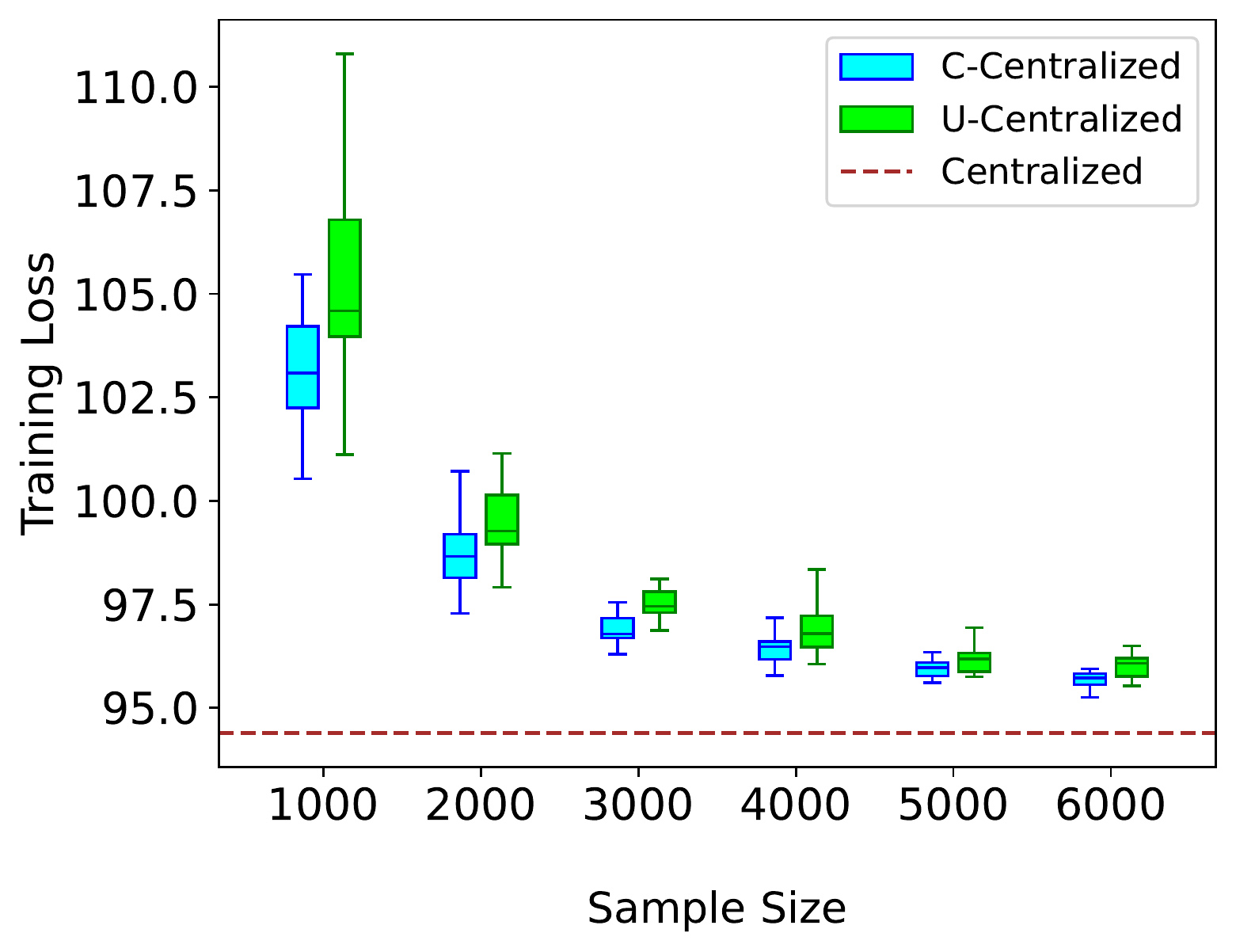}
    \end{subfigure}
    \caption{\emph{Results for Lasso regression (Section \ref{sec:exp-reg})} Left: Training loss and communication complexity of \textrm{VRLR} for different methods. C and U means using coreset or uniform sampling. The number in the parentheses denote the sample size. Right: Testing loss of \textrm{VRLR} for different methods under multiple sample sizes.}
    \label{fig:lasso-res}
\end{figure}

\begin{figure}
    \centering
    \begin{subfigure}[b]{0.525\textwidth}
        \centering
        \includegraphics[width=\textwidth]{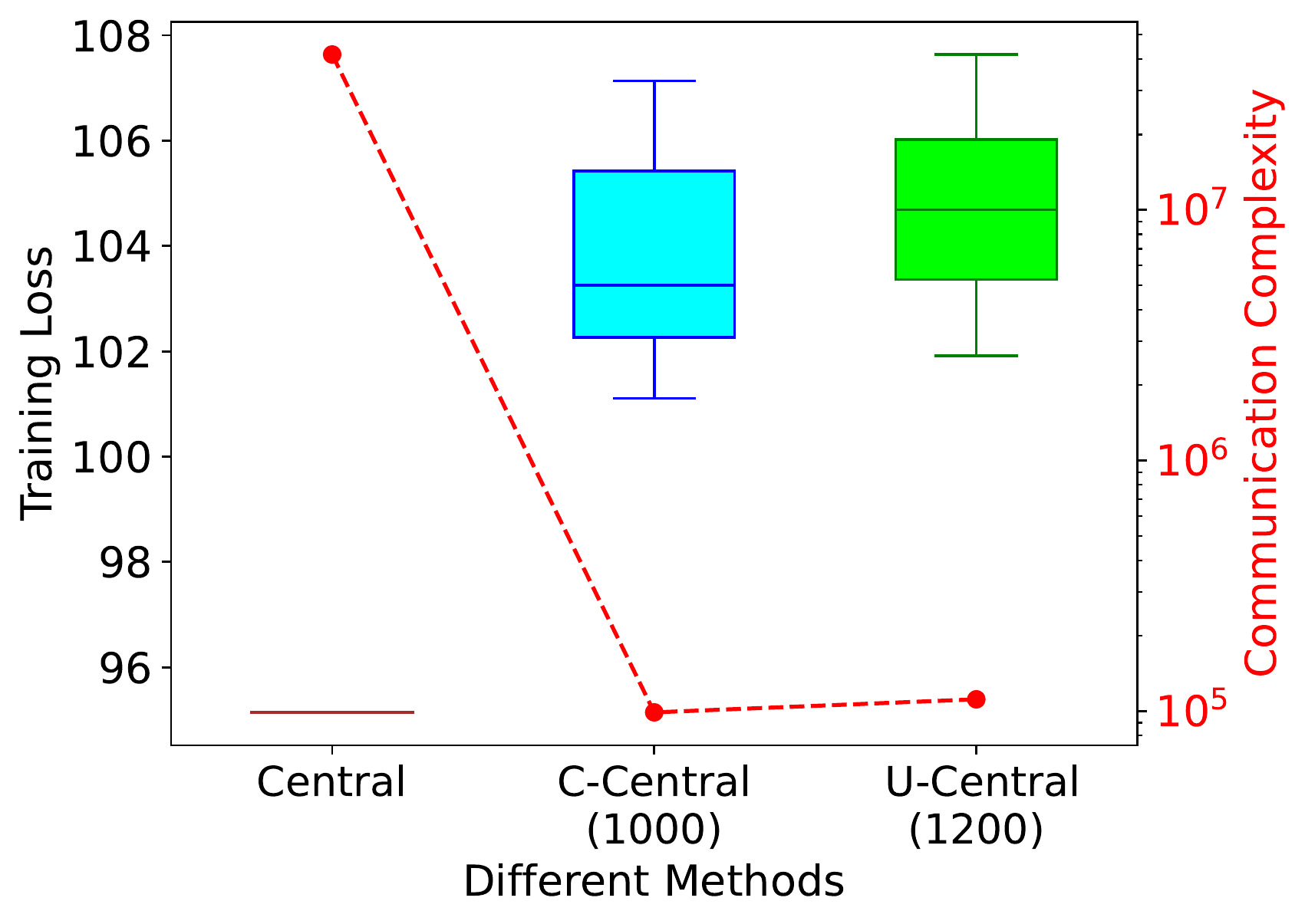}.
    \end{subfigure}
    \hfill
    \begin{subfigure}[b]{0.465\textwidth}
        \centering
        \includegraphics[width=\textwidth]{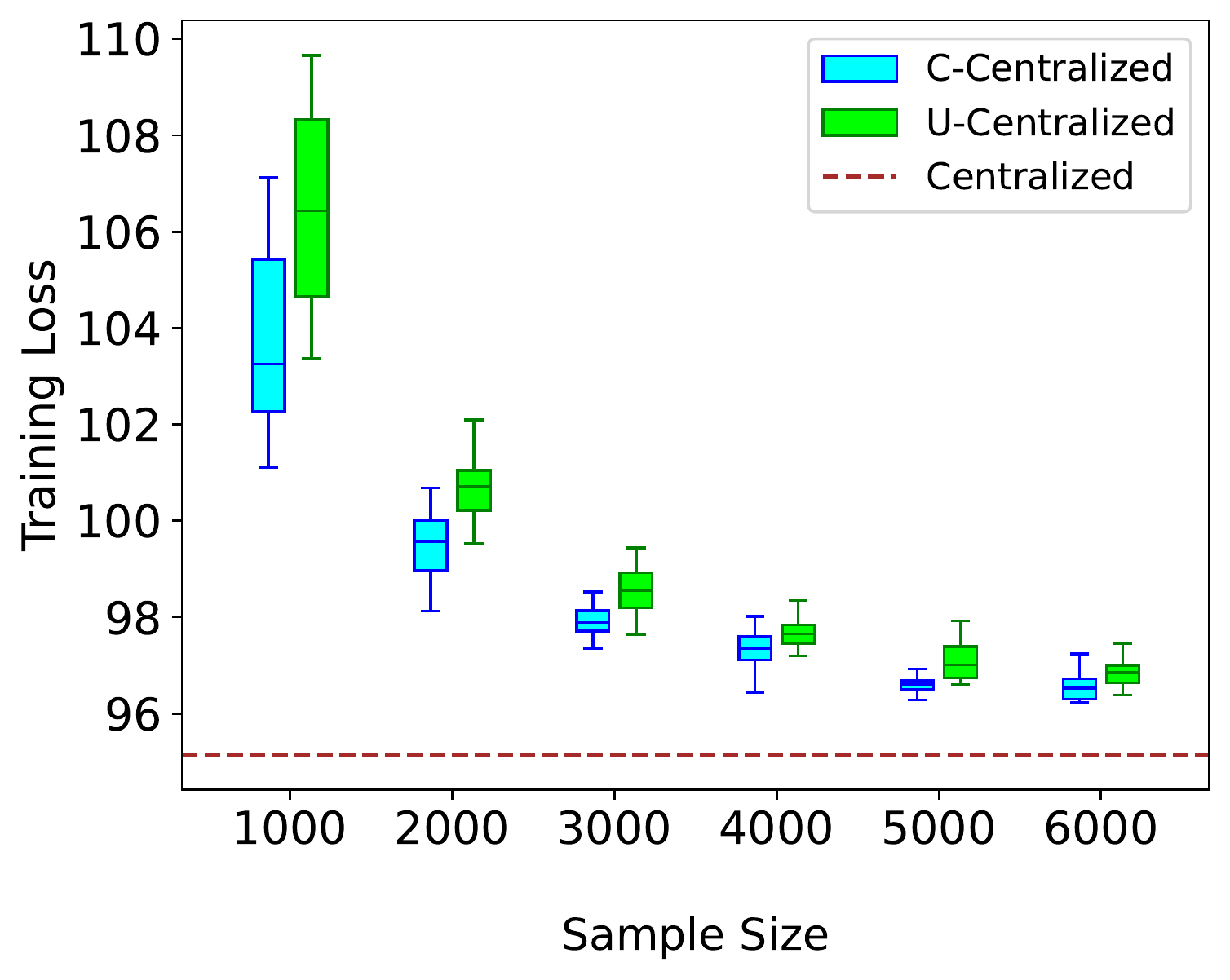}
    \end{subfigure}
    \caption{\emph{Results for elastic net (Section \ref{sec:exp-reg})} Left: Training loss and communication complexity of \textrm{VRLR} for different methods. C and U means using coreset or uniform sampling. The number in the parentheses denote the sample size. Right: Testing loss of \textrm{VRLR} for different methods under multiple sample sizes.}
    \label{fig:en-res}
\end{figure}

\subsection{Different number of centers for \textrm{VKMC}}\label{sec:exp-centers}
In this section we test our methods on \textrm{VKMC} using different number of centers.

\paragraph{Empirical setup} The experimental setup in this part is the same as the setup in Section \ref{sec:experiments} for \textrm{VKMC}, except that we are using 5 centers instead of 10 centers.

\paragraph{Empirical results} Figure \ref{fig:kmeans-res-5c} summarizes the result. All the observations in Section \ref{sec:experiments} also hold.

\begin{figure}
    \centering
    \begin{subfigure}[b]{0.525\textwidth}
        \centering
        \includegraphics[width=\textwidth]{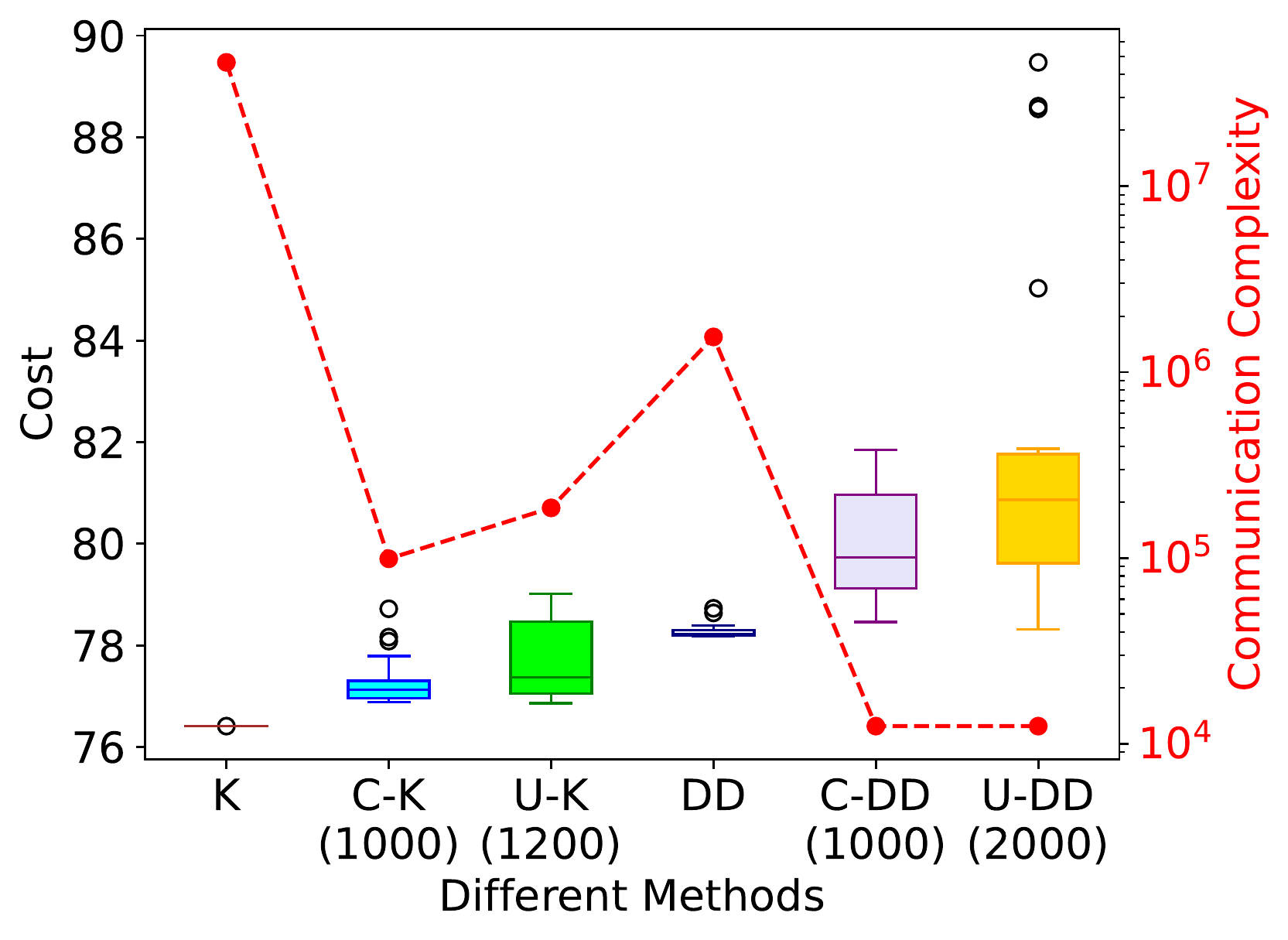}
    \end{subfigure}
    \hfill
    \begin{subfigure}[b]{0.465\textwidth}
        \centering
        \includegraphics[width=\textwidth]{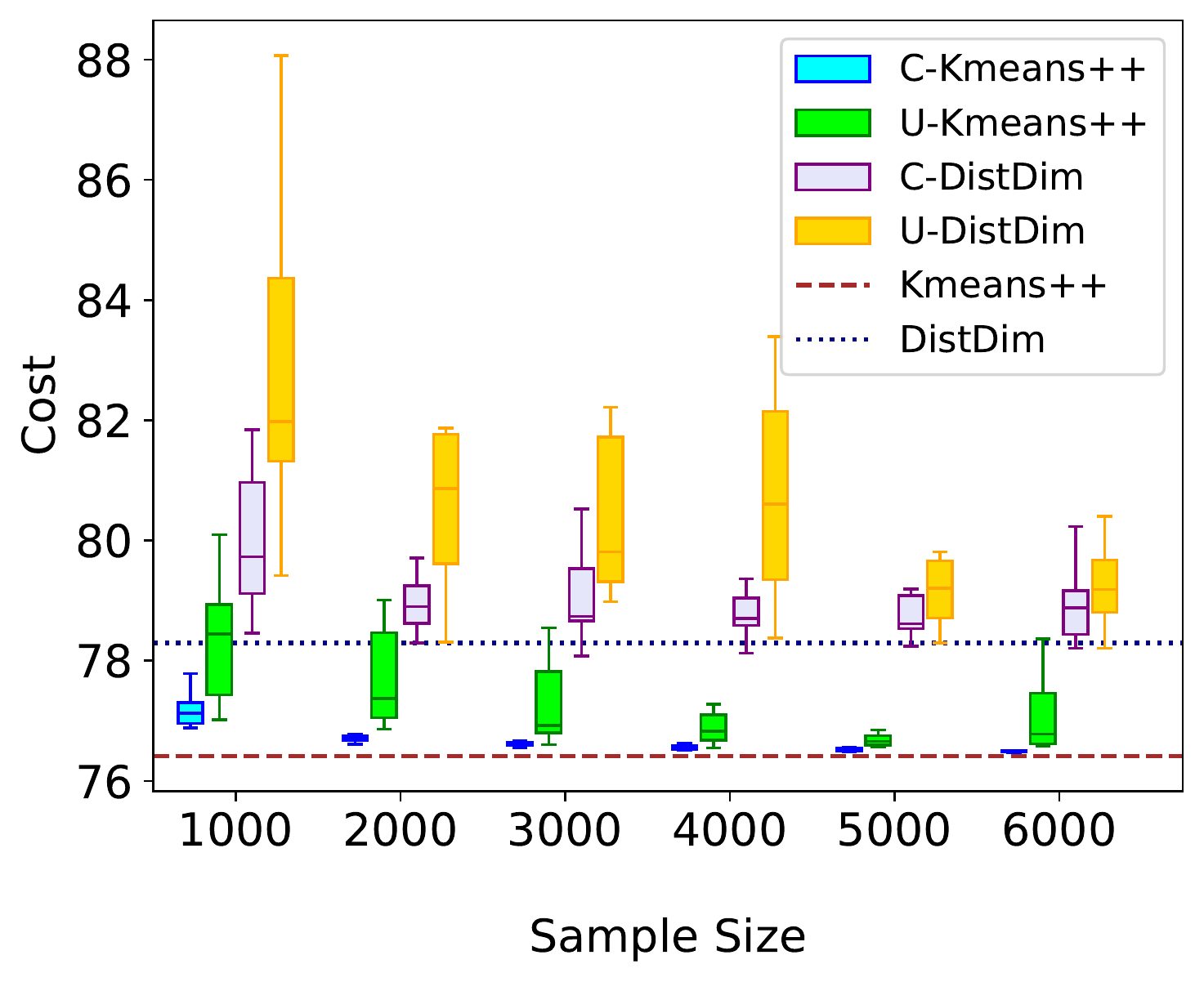}
    \end{subfigure}
    \caption{\emph{Results for \textrm{VKMC} with 5 centers (Section \ref{sec:exp-centers})} Left: Cost and communication complexity of \textrm{VKMC} for different methods. C and U means using coreset sampling or uniform sampling. The number in the parentheses denote the sample size. Right: Cost of \textrm{VKMC} for different methods under multiple sample sizes.}
    \label{fig:kmeans-res-5c}
\end{figure}

\subsection{Experiments on other datasets}\label{sec:exp-dataset}

In this section, we present the experiment results on another dataset. We choose the \dataset{KC House} Dataset~\citep{kagglekc} for both \textrm{VRLR} and \textrm{VKMC}.

\paragraph{Empirical setup} Our experiment setup is nearly the same as the setup in Section \ref{sec:experiments}. However, there are a few differences: (1) the dataset we use is \dataset{KC House} Dataset~\citep{kagglekc}, which contains 21613 data points and each datapoint constains 18 features and a label; (2) we conduct the experiment using only two parties because the limited number of features, we put the first nine features on the first party and the remaining on the second; and (3), we do not consider regularizer for \textrm{VRLR} (plain linear regression). Also note that similar to Section \ref{sec:experiments}, we normalize each feature to have standard deviation 1 during the clustering task.

\paragraph{Empirical results} For \textrm{VRLR}, we plot the training loss instead of the testing loss, since the dataset is not so large and coreset does not have theoretical guarantee for generalization error. Figure \ref{fig:ridge-res-kc} and \ref{fig:kmeans-res-kc} summarize our results for \textrm{VRLR} and \textrm{VKMC} respectively. From the results, we still find that our coreset construction method can outperform uniform sampling, and both of them can drastically reduce the communication complexity compared with the original baselines.

Note that in Figure \ref{fig:ridge-res-kc}, \algname{C-SAGA} and \algname{U-SAGA} performs much worse than the baseline \algname{Central}. However, \algname{C-Central} can perform much better and has similar performance as \algname{Central}, and this phenomenon may attribute to the fact that this problem is hard to solve by \algname{SAGA} algorithm, and using other second-order methods~\citep{yang2019quasi} may help. Also note that when the size is small (100 and 200), \algname{U-Central} may produce ``ridiculous'' solutions and the cost blows up.

\begin{figure}
    \centering
    \begin{subfigure}[b]{0.525\textwidth}
        \centering
        \includegraphics[width=\textwidth]{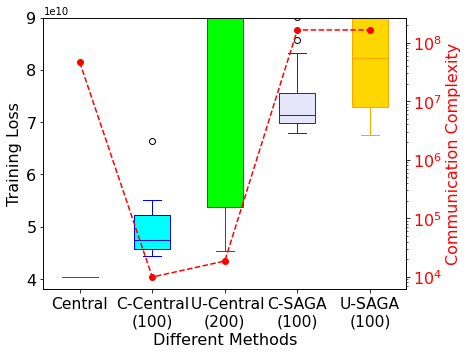}.
    \end{subfigure}
    \hfill
    \begin{subfigure}[b]{0.465\textwidth}
        \centering
        \includegraphics[width=\textwidth]{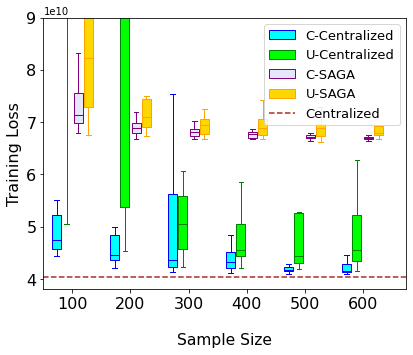}
    \end{subfigure}
    \caption{\emph{Results for \dataset{KC House} dataset (Section \ref{sec:exp-dataset})} Left: Training loss and communication complexity of \textrm{VRLR} for different methods. C and U means using coreset or uniform sampling. The number in the parentheses denote the sample size. Right: Testing loss of \textrm{VRLR} for different methods under multiple sample sizes.}
    \label{fig:ridge-res-kc}
\end{figure}

\begin{figure}
    \centering
    \begin{subfigure}[b]{0.51\textwidth}
        \centering
        \includegraphics[width=\textwidth]{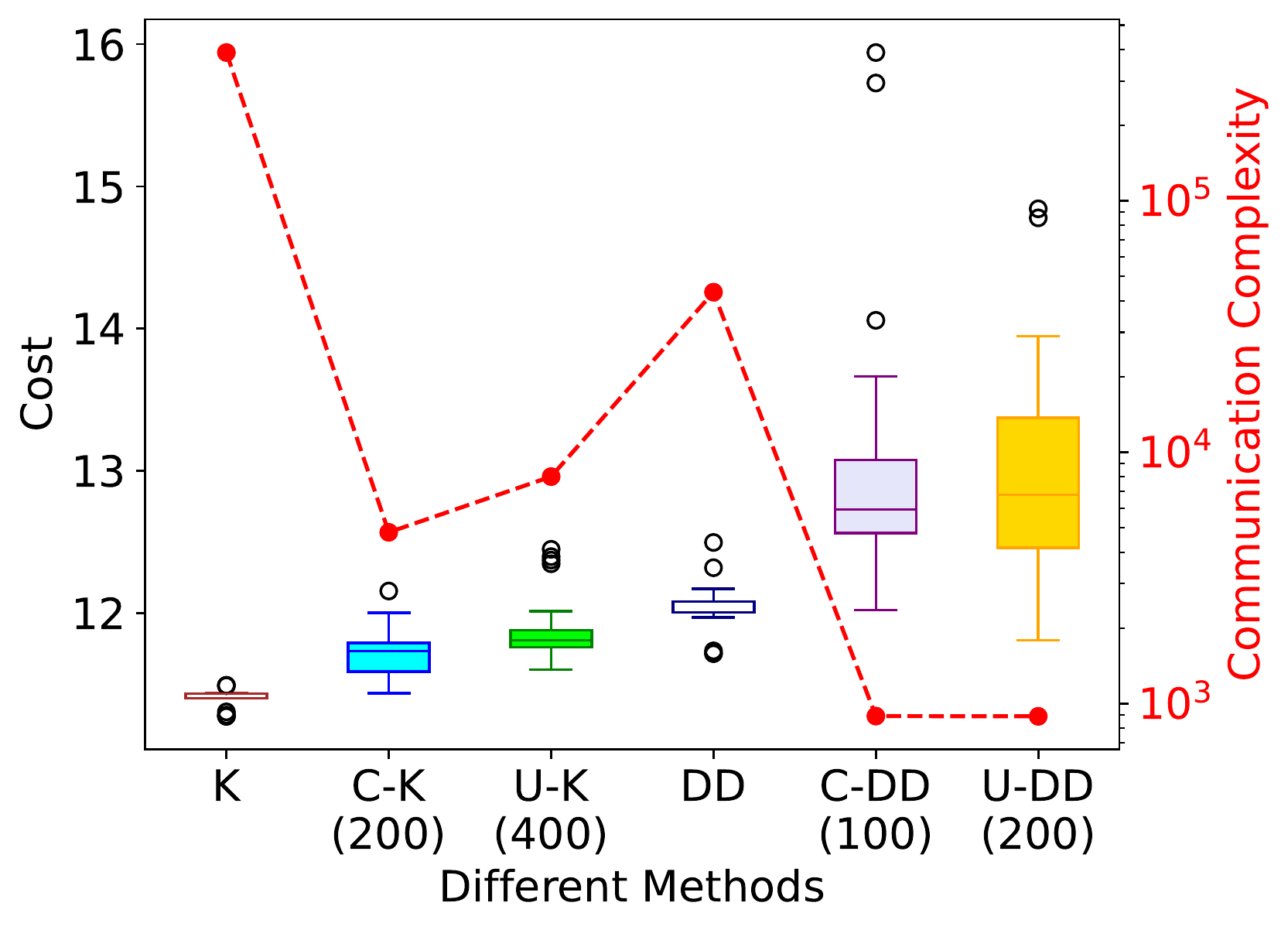}
    \end{subfigure}
    \hfill
    \begin{subfigure}[b]{0.47\textwidth}
        \centering
        \includegraphics[width=\textwidth]{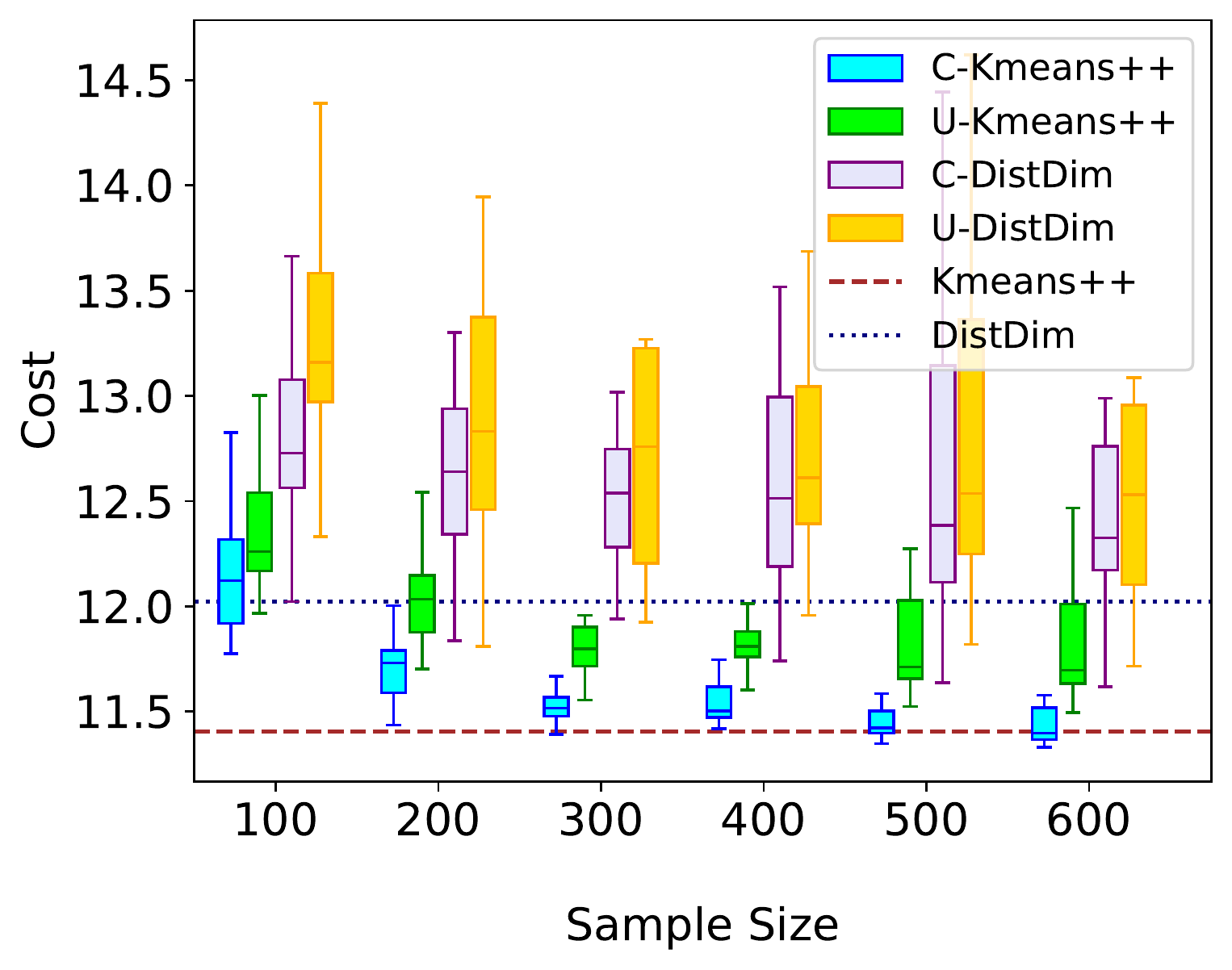}
    \end{subfigure}
    \caption{\emph{Results for \dataset{KC House} dataset (Section \ref{sec:exp-dataset})} Left: Cost and communication complexity of \textrm{VKMC} for different methods. C and U means using coreset sampling or uniform sampling. The number in the parentheses denote the sample size. Right: Cost of \textrm{VKMC} for different methods under multiple sample sizes.}
    \label{fig:kmeans-res-kc}
\end{figure}

\section{Justification of Data Assumptions}\label{sec:justify}
In this section, we justify our data assumptions in Section \ref{sec:vrlr} (Assumption \ref{ass:vrlr}) and Section \ref{sec:vkmc} (Assumption \ref{ass:vkmc}). We show that in the \textit{smoothed analysis} regime, Assumption \ref{ass:vrlr} and \ref{ass:vkmc} are easy to satisfy with some standard assumptions. In Section \ref{sec:justify-vrlr}, we show the results related to Assumption \ref{ass:vrlr}, and in Section \ref{sec:justify-vkmc}, we justify Assumption \ref{ass:vkmc}.

\subsection{Justification of Assumption \ref{ass:vrlr}}\label{sec:justify-vrlr}
In this section, we interpret and justify Assumption \ref{ass:vrlr}. First, we recall Assumption \ref{ass:vrlr}.

\assvrlr*

Assumption \ref{ass:vrlr} requires that the subspace generated by any party cannot be included in the subspace generated by all other parties. However, it it not sure what standard assumptions can lead to Assumption \ref{ass:vrlr}. The following lemma shows that, $\sigma_{\min}(\mU)$ can be lower bounded by th smallest and largest singular value of matrix $\mX' = [\mX, \vy]$.

\begin{lemma}\label{lem:ass-vrlr-help}
    If matrix $\mX' = [\mX, \vy]$ has smallest singular value $\sigma_{\min}(\mX') > 0$ and largest singular value $\sigma_{\max}(\mX')$, we have
    \[\sigma_{\min}(\mU) \ge \frac{\sigma_{\min}(\mX')}{\sigma_{\max}(\mX')}.\]
\end{lemma}

\begin{proof}
    Because we assume $\mX'$ has smallest singular value, we can represent $\mX' = \mU \mA$, where $\mA$ is a $d+1$ by $d+1$ matrix with rank $d+1$.
    
    Now for any $\vw$, we have
    \[\norm{\mU \vw} = \norm{\mX'\mA^{-1}\vw} \ge \sigma_{\min}(\mX')\norm{\mA^{-1}\vw}.\]
    Note that $\mA$ has rank $d+1$, and thus $\sigma_{\min}(\mA^{-1}) = \nicefrac{1}{\sigma_{\max}(\mA)}$. Besides, $\mA = \text{diag}(\mA^{(1)},\dots,\mA^{(T)})$ is a block diagonal matrix, where $\mX^{(j)} = \mA^{(j)}\mU^{(j)}$ for $j\in [T-1]$ and $[\mX^{(T)},\vy] = \mA^{(T)}\mU^{(T)}$, and thus $\sigma_{\max}(\mA) = \max_{j\in [T]}\{\sigma_{\max}(\mA^{(j)})\}$. Because $\mU^{(j)}$ is the orthonormal basis of $\mX^{(j)}$ or $[\mX^{(T)},\vy]$, we have
    \[\sigma_{\max}(\mA^{(j)}) = \sigma_{\max}(\mX^{(j)}), \quad \sigma_{\max}(\mA^{(T)}) = \sigma_{\max}([\mX^{(T)},\vy]).\]
    We also have $\sigma_{\max}(\mX') \ge \sigma_{\max}(\mX^{(j)})$ and $\sigma_{\max}(\mX') \ge \sigma_{\max}([\mX^{(T)},\vy])$. Combining all the properties together, we get $\sigma_{\max}(\mA) \le \sigma_{\max}(\mX')$, and thus conclude the proof.
\end{proof}

Using the preivous lemma, it is easy to analyze the smallest singular value $\sigma(\mU)$ in the smoothed analysis regime. Specifically, we prove that for any dataset $[\mX,\vy]$ satisfying certain conditions, we add a random perturbation on the dataset, resulting $[\mX_p,\vy_p]$, and we show that with high probability, $\mU_p$ (which is constructed from dataset $[\mX_p, \vy_p]$ has smallest singular value. The result is formalized in the following theorem.

\begin{theorem}\label{thm:smoothed-analysis-vrlr-ass}
    There exists constant $n_0$ such that for any dataset $[\mX,\vy]\in \R^{n\times (d+1)}$ where each data point $\norm{[\vx_i; y_i]}_2^2 \le B$ and $n \ge 2d, n \ge n_0$. If we perturb the dataset by a small random Gaussian noise $[\mX_p.\vy_p]$ where $\mX_p = \mX+\mZ$, $\vy_p = \vy+\vw$, and each coordinate of $\mZ$ and $\vw$ comes from $\gN(0,r^2B^2)$, then with high probability, the basis $\mU_p$ computed from $[\mX_p.\vy_p]$ has smallest singular value at least $\Omega(r)$.
\end{theorem}

In order to prove Theorem \ref{thm:smoothed-analysis-vrlr-ass}, we use the following theorem (Theorem 1.1 in \citep{burgisser2010smoothed}).

\begin{proposition}[Smoothed analysis of condition number, Theorem 1.1 in \citep{burgisser2010smoothed}]\label{prop:smoothed-analysis-cond}
    Suppose that $\bar\mA \in \R^{n\times d}$ satisfies $\norm{\bar\mA} \le 1$, and let $0 < r_p\le 1$. Then,
    \[\Pr_{\mA\sim \gN(\bar\mA,r^2\mI)}\left\{\kappa(\mA) \ge C_1 t\right\} \le \left(\nicefrac{C_2}{t} + \nicefrac{C_2}{r_p\sqrt{n} t}\right)^{n-d+1},\]
    for some constants $C_1, C_2, C_3$ and all $t \ge C_3$.
\end{proposition}

Roughly speaking, Proposition \ref{prop:smoothed-analysis-cond} claims that with high probability, the condition number under the smoothed analysis regime should be bounded above. Then with the help of Lemma \ref{lem:ass-vrlr-help} and Proposition \ref{prop:smoothed-analysis-cond}, we can now prove Theorem \ref{thm:smoothed-analysis-vrlr-ass}.

\begin{proof}[Proof of Theorem \ref{thm:smoothed-analysis-vrlr-ass}]
    For simplicity, we treat denote $\mD = [\mX,\vy]$ and $\mD_p = [\mX_p,\vy_p]$, and $\mD_p = \mD+\mA$, where each coordinate of $\mA$ comes form $\gN(0,r^2B^2)$.
    
    Note that the condition number of a matrix is `scale invariant', which means that
    \[\kappa(\mA) = \kappa(c\mA),\]
    for constants $c\neq 0$.
    
    Now, since the row of $\mD$ has bounded norm $B$, thus $\norm{\mD} \le B\sqrt{n}$. By the scale invariance of condition number, we have
    \[\kappa(\mD_p) = \kappa(\mD_p/(B\sqrt{n})).\]
    Now, the perturbation factor $r_p$ in Proposition \ref{prop:smoothed-analysis-cond} is $\nicefrac{rB}{B\sqrt{n}} = \nicefrac{r}{\sqrt{n}}$, and we know that
    \[\Pr\left\{\kappa(\mD_p) \ge \nicefrac{C_1}{r}\right\} \le \left(C_2r + C_2\right)^{n-d+1},\]
    for some constants $C_1, C_3 > 0$, constant $C_2$ s.t. $0 < C_2 < 1$ and all $r \le C_3$. Directly applying Lemma \ref{lem:ass-vrlr-help} concludes the proof.
\end{proof}

\subsection{Justification of Assumption \ref{ass:vkmc}}\label{sec:justify-vkmc}

In this section, we justify Assumption \ref{ass:vkmc}. We first recall the assumption.

\assvkmc*

Roughly speaking, this assumption requires there is a party that is ``important'', and any two data points which can be differentiated can also be differentiated on that party to some extent. In reality, this assumption should be approximately satisfied since different features should be ``correlated''.

Next, similar to the justification of Assumption \ref{ass:vrlr}, we use smoothed analysis framework to show that for dataset $\mX$ under certain conditions, by perturbing the dataset for a little bit, Assumption \ref{ass:vkmc} will be satisfied with high probability. Formally, we have the following theorem.

\begin{theorem}\label{thm:smoothed-analysis-vkmc-ass}
    For any dataset where each data point $\norm{\vx_i}_2^2 \le B$ for all $\vx_i\in\mX$ and $\max_{j\in [T]} d_j \ge \Omega(\log^2 n)$. If we perturb the dataset by a small random Gaussian noise $\mX_p$ where $\mX_p = \mX+\mZ$, and each coordinate of $\mZ$ and $\vw$ comes from $\gN(0,r^2B^2)$. Then with high probability, $\mX_p$ satisfies Assumption \ref{ass:vkmc} with
    \[\tau = O\left(\frac{1}{r^2} + \frac{d}{\log^2 n}\right)\]
\end{theorem}

The intuition of the proof is that, the norm of a high-dimensional (sub-)gaussian random vector should concentrate around $\Theta(\sqrt{d})$, where $d$ is the dimension of the (sub-)gaussian random vector. Thus, as long as we add some perturbation to the original dataset, the norm of the difference between any two perturbed data points on party $j$ should be at least $\sqrt{d_j}$. Formally, we have the following proposition for the concentration of norm.

\begin{proposition}[Concentration of the norm]\label{prop:norm-concentration}
    Let $\vxi = (\vxi_1, \dots,\vxi_d)\in\R^d$ be a random Gaussian vector, where each coordinate is sampled from $\gN(0,r^2)$ independently. Then there exists constants $c$ such that for any $t \ge 0$,
    \[\Pr\left\{\left|\norm{\vxi}_2 - r\sqrt{d}\right| \ge rt\right\} \le 2\exp\left(-ct^2\right)\]
\end{proposition}

Now with the help of this proposition, we can now prove Theorem \ref{thm:smoothed-analysis-vkmc-ass}.

\begin{proof}[Proof of Theorem \ref{thm:smoothed-analysis-vkmc-ass}]
    First, we upper bound $\norm{\tilde \vx_i - \tilde \vx_j}^2$ where $\tilde \vx_i$ denote the $i$-th perturbed data and we use $\vxi_i = \tilde \vx_i - \vx_i$ to denote the random perturbation. We have
    \[\norm{\tilde \vx_i - \tilde \vx_j}^2 \le 2\norm{\vx_i - \vx_j}^2 + 2\norm{\vxi_i-\vxi_j}^2.\]
    From the assumption, we know that $\norm{\vx_i - \vx_j} \le 2B$, and thus we only need to bound the second term.
    From Proposition \ref{prop:norm-concentration}, we know that for fixed $i \neq j$, we have
    \[\Pr\left\{\left|\norm{\vxi_i - \vxi_j} -  \sqrt{2}rB\sqrt{d}\right| \ge c rB\log n\right\} \le 2\exp\left(4 \log n\right),\]
    for some constants $c$ since $\vxi_i -\vxi_j$ is a Gaussian random vector whose entries are drawn from $\gN(0,2r^2B^2)$. Thus, with probability at least $1-\nicefrac{2}{n^4}$, we have
    \[\norm{\tilde \vx_i - \tilde \vx_j}^2 \le 8B^2 + 4r^2B^2 d + c r^2B^2\log^2 n,\]
    for some constant $c$. Then applying the union bound, we know that with probability at least $1-\frac{1}{n^2}$, 
    \[\norm{\tilde \vx_i - \tilde \vx_j}^2 \le 8B^2 + c r^2B^2\log^2 n, \forall i\neq j,\]
    for some constant $c$. Without loss of generality, suppose that $d_1 = \max_{j\in [T]} d_j$, and then we lower bound $\norm{\tilde\vx_i^{(1)} - \tilde\vx_j^{(1)}}^2$. First since $\norm{\tilde\vx_i^{(1)} - \tilde\vx_j^{(1)}}^2$ is the noncentralized $\chi^2$ distribution, we have
    \[\Pr\left\{\norm{\tilde\vx_i^{(1)} - \tilde\vx_j^{(1)}}^2 \ge t\right\} \ge \Pr\left\{\norm{\vxi_i^{(1)} - \vxi_j^{(1)}}^2 \ge t\right\}.\]
    Then from Proposition \ref{prop:norm-concentration}, we have
    \[\Pr\left\{\left|\norm{\vxi_i^{(1)} - \vxi_j^{(1)}} -  \sqrt{2}rB\sqrt{d_1}\right| \ge c rB\log n\right\} \le 2\exp\left(4 \log n\right),\]
    for some constant $c$. Thus, if $d_1 \ge C\log^2n$ for some large enough constant $c$, we know that with probability at least $1-\nicefrac{2}{n^4}$, 
    \[\norm{\tilde \vx_i^{(1)} - \tilde \vx_j^{(1)}}^2 \ge c r^2B^2 \log^2 n,\]
    for some constant $c$. Then with a union bound, we know that with probability at least $1-\nicefrac{1}{n^2}$,
    \[\norm{\tilde \vx_i^{(1)} - \tilde \vx_j^{(1)}}^2 \ge c r^2B^2 \log^2 n, \forall i\neq j,\]
    for some constant $c$. Combining with the previous part, we know that if $\max_{j\in [T]}d_j \ge C \log^2 n$ for some large constant $C$, then with probability at least $1-\frac{1}{n}$, we have
    \[\frac{\norm{\tilde \vx_i - \tilde \vx_j}^2}{\norm{\tilde \vx_i^{(1)} - \tilde \vx_j^{(1)}}^2} \le O\left(\frac{B^2 + r^2B^2d + r^2 B^2 \log^2 n }{r^2 B^2 \log^2 n}\right) = O\left(\frac{1}{r^2} + \frac{d}{\log^2 n}\right).\]
\end{proof}

\section{Proof of Theorem~\ref{thm:coreset_reduce}}
\label{sec:coreset_reduce}

\begin{proof}[Proof of Theorem~\ref{thm:coreset_reduce}]
    We only take \textrm{VRLR} as an example.
    We consider the following communication scheme: First apply the communication scheme $A'$ to construct an $\eps$-coreset $(S,w)$ for \textrm{VRLR} in the server; then the server broadcasts $(S,w)$ to all parties; and finally apply the communication scheme $A$ to $(S,w)$ and obtain a solution $\vtheta\in \R^d$ in the server.
    
    Let $\vtheta^\star$ be the optimal solution for the offline regularized linear regression problem.

    By the coreset definition, we have that 
    \begin{align*}
    \costR(X,\vtheta) &\leq && (1+\eps) \costR(S, \vtheta) & (\text{by coreset definition}) \\
    & \leq && (1+\eps)\alpha\cdot \costR(S, \vtheta^\star) & (\text{by $A$}) \\
    & \leq && (1+\eps)^2 \alpha\cdot \costR(X, \vtheta^\star) & (\text{by coreset definition}) \\
    & \leq && (1+3\eps) \alpha\cdot \costR(X, \vtheta^\star), & (\eps\in (0,1))
    \end{align*}
    which proves the approximation ratio.

    For the total communication complexity, note that the broadcasting step costs $2Tm$.
    This completes the proof.
\end{proof}

\section{Proof of Theorem~\ref{thm:meta_alg}}
\label{sec:proof_meta}

For preparation, we first introduce a well-known importance sampling framework for offline coreset construction by~\cite{feldman2011unified,braverman2016new}.

\begin{theorem}[\bf{\FL\ framework~\cite{feldman2011unified,braverman2016new}}]
    \label{thm:fl}
    Let $\eps, \delta\in (0,1/2)$ and let $k\geq 1$ be an integer.
    Let $\mX\subset \R^d$ be a dataset of $n$ points together with a label vector $\vy\in \R^n$, and $\vg\in \R_{\geq 0}^n$ be a vector.
    Let $\gG:= \sum_{i\in [n]} g_i$.
    Let $S\subseteq [n]$ be constructed by taking $m\geq 1$ samples, where each sample $i \in [n]$ is selected with probability $\frac{g_i}{\gG}$ and has weight $w(i):=\frac{\gG}{|S| g_i}$.
    Then we have
    \begin{itemize}
        \item If $g_i \geq \sup_{\vtheta\in \R^d}\frac{\costR_i(\mX,\vtheta)}{\costR(\mX,\vtheta)}$ holds for any $i\in [n]$ and $m = O\left(\eps^{-2}\gG(d^2\log \gG + \log (1/\delta)) \right)$, with probability at least $1-\delta$, $(S,w)$ is an $\eps$-coreset  for offline regularized linear regression.
        \item If $g_i \geq \sup_{\mC\in \gC}\frac{\costC_i(\mX,\mC)}{\costC(\mX,\mC)}$ holds for any $i\in [n]$ and $m = O\left(\eps^{-2}\gG(dk\log \gG + \log (1/\delta)) \right)$, with probability at least $1-\delta$, $(S,w)$ is an $\eps$-coreset for offline \kMeans\ clustering.
    \end{itemize}
\end{theorem}

\noindent
We call $g_i$ the sensitivity of point $\vx_i$ that represents the maximum contribution of $\vx_i$ over all possible parameters, and call $\gG$ the total sensitivity.
By~\cite{varadarajan2012sensitivity}, we note that the total sensitivity can be upper bounded by $O(d)$ for offline regularized linear regression and by $O(k)$ for offline \kMeans\ clustering. 
By the \FL\ framework, it suffices to compute a sensitivity vector $\vg\in \R^n$ for offline coreset construction.

\begin{proof}[Proof of Theorem~\ref{thm:meta_alg}]
    We first discuss the communication complexity of Algorithm~\ref{alg:dis}.
    At the first round, the communication complexity in Line 2 is $T$ and in Line 4 is $T$.
    At the second round, the communication complexity in Line 5 is at most $\sum_{j\in [T]} a_j = m$ and in Line 6 is at most $m T$.
    At the third round, the communication complexity in Line 7 is at most $m T$.
    Overall, the total communication complexity is $O(mT)$.

    Next, we prove the correctness.
    We only take \textrm{VRLR} as an example and the proof for \textrm{VKMC} is similar.
    Note that each sample in $S$ is equivalent to be drawn by the following procedure: Sample $i\in [n]$ with probability $\nicefrac{\sum_{j\in [T]}g_i^{(j)}}{\gG}$.
    This is because by Lines 3 and 5, the sampling probability of $i\in [n]$ is exactly
    \[
    \sum_{j\in [T]} \frac{\gG^{(j)}}{\gG} \cdot \frac{g_i^{(j)}}{\gG^{(j)}} = \frac{\sum_{j\in [T]}g_i^{(j)}}{\gG}.
    \]
    Then letting $g_i' = \zeta \cdot \sum_{j\in [T]}g_i^{(j)}$ for each $i\in [n]$, we have 
    \[
    g_i'\geq \sup_{\vtheta\in\R^d}\frac{\costR_i(\mX,\vtheta)}{\costR(\mX,\vtheta)}
    \]
    by assumption.
    This completes the proof by plugging $g_i'$ to Theorem~\ref{thm:fl}.
\end{proof}

\section{Omitted Proof in Section \ref{sec:vrlr}}
\label{sec:proof_vrlr}

\subsection{Communication lower bound for \textrm{VRLR} coreset construction}
\label{sec:lowerbound_vrlr}
The proof is via a reduction from an EQUALITY problem to the problem of coreset construction for \textrm{VRLR}.
For preparation, we first introduce some concepts in the field of communication complexity.
\paragraph{Communication complexity.} Here it suffices to consider the two-party case ($T=2$).
Assume we have two players Alice and Bob, whose inputs are $x \in \gX$ and $y \in \gY$ respectively.
They exchange messages with a coordinator according to a protocol $\Pi$ (deterministic/randomized) to compute some function $f: \gX\times\gY\rightarrow\gZ$.
For the input $(x,y)$, the coordinator outputs $\Pi(x,y)$ when Alice and Bob run $\Pi$ on it.
We also use $\Pi(x,y)$ to denote the transcript (concatenation of messages).
Let $|\Pi_{x,y}|$ be the length of the transcript.
The communication complexity of $\Pi$ is defined as $\max_{x,y}|\Pi_{x,y}|$.
If $\Pi$ is a randomized protocol, we define the \textit{error} of $\Pi$ by $\max_{x,y} \sP(\Pi(x,y) \neq f(x,y))$, where the max is over all inputs $(x,y)$ and the probability is over the randomness used in $\Pi$.
The \textit{$\delta$-error randomized communication complexity} of $f$, denoted by $R_\delta(f)$, is the minimum communication complexity of any protocol with error at most $\delta$.
\paragraph{EQUALITY problem.} In the EQUALITY problem, Alice holds $a=\{a_1,\ldots,a_n\} \in \{0,1\}^n$ and Bob holds $b=\{b_1,\ldots,b_n\}\in\{0,1\}^n$.
The goal is to compute $\textrm{EQUALITY}(a,b)$ which equals 1 if $a_i=b_i$ for all $i\in[n]$ otherwise 0.
The following lemma gives a well-known lower bound for deterministic communication protocols that correctly compute \textrm{EQUALITY} function.

\begin{lemma}[\bf{Communication complexity of \textrm{EQUALITY}~\citep{kushilevitz1997communication}}]
    \label{lemma:comm_equality}
    The deterministic communication complexity of \textrm{EQUALITY} is $\Omega(n)$.
\end{lemma}

\paragraph{Reduction from \textrm{EQUALITY}.} Now we are ready to prove Theorem~\ref{thm:hardness-vrlr}.
\begin{proof}[Proof of Theorem~\ref{thm:hardness-vrlr}]
    We prove this by a reduction from \textrm{EQUALITY}.
    For simplicity, it suffices to assume $d=1$ and $T=2$ in the \textrm{VRLR} problem.
    Given an \textrm{EQUALITY} instance of size $n$, let $a\in\{0,1\}^n$ be Alice's input and $b\in\{0,1\}^n$ be Bob's input.
    They construct inputs $\mX \in \R^{n}$ and $\vy \in \R^n$ for \textrm{VRLR}, where $\mX=a$ and $\vy=b$.
    We denote $S \subseteq [n]$ with a weight function $w:S\rightarrow \R_{\geq 0}$ to be an $\eps$-coreset such that for any $\vtheta \in \R$, we have
    \[
    \costR(S,\vtheta):=\sum_{i\in S} w(i)\cdot (\vx_i^\top \vtheta - y_i)^2 + R(\vtheta) \in (1\pm \eps)\cdot\costR(\mX, \vtheta).
    \]
    Based on the above guarantee, w.l.o.g, if we set $\theta=1$ and $R=0$, then there exist two cases with positive cost: $(a_i,b_i)=(0,1)$ or $(1,0)$.
    In other words, $\textrm{EQUALITY}(a,b)=0$ if and only if the set $\{(\vx_i, y_i): i\in S\}$ includes $(0,1)$ or $(1,0)$.
    Thus, any deterministic protocol for \textrm{VRLR} coreset construction can be used as a deterministic protocol for \textrm{EQUALITY}.
    The lower bound follows from Lemma~\ref{lemma:comm_equality}.
\end{proof}

\subsection{Proof of Theorem~\ref{thm:coreset-vrlr}}
In this section, we show the detailed proof of Theoem~\ref{thm:coreset-vrlr}. The proof idea is to bound the sensitivity of each data point and then apply Theorem \ref{thm:meta_alg}. Recall that in Theorem \ref{thm:meta_alg}, we define
\[\zeta = \max_{i\in [n]} \nicefrac{\sup_{\vtheta\in \R^d}\frac{\costR_i(\mX,\vtheta)}{\costR(\mX,\vtheta)}}{\sum_{j\in [T]} g_i^{(j)}}.\]

We first show the following main lemma.

\begin{lemma}\label{lem:vrlr-sensitivity}
    Under Assumption \ref{ass:vrlr}, the sensitivity of a data point can be bounded by
    \[\sup_{\vtheta\in \R^d}\frac{\costR_i(\mX,\vtheta)}{\costR(\mX,\vtheta)} \le \frac{g_i}{\gamma^2},\]
    which means that $\zeta \le \nicefrac{1}{\gamma^2}$.
\end{lemma}

\begin{proof}
The sensitivity function for each data point $(\vx_i,y_i)$ is defined as
\[\sup_{\vtheta\in \R^d}\frac{\costR_i(\mX,\vtheta)}{\costR(\mX,\vtheta)} = \sup_{\vtheta}\frac{(\vx_i^\top \vtheta - y_i)^2 + \frac{\lambda R(\vtheta)}{n}}{\norm{\mX \vtheta - \vy}^2 + \lambda R(\vtheta)}.\]

First, we have
\begin{align*}
    \sup_{\vtheta\in \R^d}\frac{(\vx_i^\top \vtheta - y_i)^2 + \frac{\lambda R(\vtheta)}{n}}{\norm{\mX \vtheta - \vy}^2 + \lambda R(\vtheta)} & = \sup_{\vtheta\in \R^d}\left(\frac{(\vx_i^\top \vtheta - y_i)^2}{\norm{\mX \vtheta - \vy}^2 + \lambda R(\vtheta)} + \frac{\frac{\lambda R(\vtheta)}{n}}{\norm{\mX \vtheta - \vy}^2+\lambda R(\vtheta)}\right) \\
    & \le \sup_{\vtheta\in \R^d}\left(\frac{(\vx_i^\top \vtheta - y_i)^2}{\norm{\mX \vtheta - \vy}^2} + \frac{1}{n}\right),
\end{align*}
where we separate the regression loss and the regularized loss.

Then for the regression loss, define $\mX' = [\mX, \vy]$ and $d' = \sum_{j\in [T]} d'_j$, we have
\begin{align*}
    \sup_{\vtheta\in \R^d}\frac{(\vx_i^\top \vtheta - y_i)^2}{\norm{\mX \vtheta - \vy}^2}  \le \sup_{\vtheta\in \R^{d+1}}\frac{((\vx'_i)^\top \vtheta)^2}{\norm{\mX' \vtheta}^2} = \sup_{\vtheta\in \R^{d'}}\frac{((\vu_i)^\top \vtheta)^2}{\norm{\mU \vtheta}^2}
\end{align*}

Note that under Assumption \ref{ass:vrlr}, matrix $\mU$ has smallest singular value $\sigma_{\min} \ge \gamma > 0$, and we can get
\begin{align*}
    \sup_{\vtheta\in\R^d}\frac{(\vx_i^\top \vtheta - y_i)^2}{\norm{\mX \vtheta - \vy}^2} & \le \sup_{\vtheta\in\R^{d'}}\frac{((\vu_i)^\top \vtheta)^2}{\norm{\mU \vtheta}^2} \\
    & \le \sup_{\vtheta\in\R^{d'}}\frac{((\vu_i)^\top \vtheta)^2}{\sigma_{\min}^2 \norm{\vtheta}^2} \\
    & \le \frac{\norm{\vu_i}^2}{\gamma^2} \\
    & = \frac{\sum_{j\in [T]}\norm{\vu_i^{(j)}}^2}{\gamma^2}.
\end{align*}

Recall that $g_i=\sum_{j\in[T]}g_i^{(j)}=\sum_{j\in[T]} \norm{\vu_i^{(j)}}^2 + \frac{T}{n}$.
Hence,
\[\sup_{\vtheta\in \R^d}\frac{\costR_i(\mX,\vtheta)}{\costR(\mX,\vtheta)} \leq \frac{\sum_{j\in [T]}\norm{\vu_i^{(j)}}^2}{\gamma^2} + \frac{1}{n} \leq \frac{g_i}{\gamma^2}.\]
\end{proof}

Now with the help of Lemma \ref{lem:vrlr-sensitivity}, we can prove Theorem \ref{thm:coreset-vrlr}.

\begin{proof}[Proof of Theorem \ref{thm:coreset-vrlr}]
Note that from Lemma \ref{lem:vrlr-sensitivity}, we know that $\zeta\le\nicefrac{1}{\gamma^2}$. Also note that from Algorithm \ref{alg:coreset-vrlr}, we have
\begin{align*}
    \gG = \sum_{j\in [T]}\sum_{i\in [n]}\left(\norm{\vu_i^{(j)}}^2 + \frac{1}{n}\right) = \sum_{j\in [T]}\normF{\mU^{(j)}}^2 + T= \sum_{j\in [T]} d'_j + T \le d+T+1 \le 2d+1.
\end{align*}

Then we apply Theorem \ref{thm:meta_alg}, the $\eps$-coreset size for \textrm{VRLR} can be bounded by
\[m = O(\eps^{-2}\gamma^{-2}d(d^2\log{(\gamma^{-2}d)}+\log{1/\delta})),\]
and the communication complexity is $O(mT)$.
\end{proof}

\section{Omitted Proof in Section \ref{sec:vkmc}}
\label{sec:proof_vkmc}

\subsection{Communication lower bound for \textrm{VKMC} coreset construction}
\label{sec:lowerbound_vkmc}
The proof is via a reduction from a set-disjointness (DISJ) problem to the problem of coreset construction for \textrm{VKMC}.
\paragraph{DISJ problem.} In the DISJ problem, Alice holds $a=\{a_1,\ldots,a_n\} \in \{0,1\}^n$ and Bob holds $b=\{b_1,\ldots,b_n\} \in \{0,1\}^n$.
The goal is to compute $\textrm{DISJ}(a,b)=\bigvee_{i \in [n]} (a_{i} \bigwedge b_{i})$.
\noindent
The following lemma gives a well-known communication lower bound for \textrm{DISJ}.
\begin{lemma}[\bf{Communication complexity of \textrm{DISJ}~\citep{kalyanasundaram1992probabilistic,razborov1992distributional,bar2004information}}]
    \label{lemma:comm_disj}
    The randomized communication complexity of \text{DISJ} is $\Omega(n)$,
    i.e., for $\delta \in [0, 1/2)$ and $n \geq 1$, $R_\delta(\textrm{DISJ})=\Omega(n)$.
\end{lemma}

\paragraph{Reduction from DISJ.} Now we are ready to prove Theorem~\ref{thm:hardness-vkmc}.
\begin{proof}[Proof of Theorem~\ref{thm:hardness-vkmc}]
    We prove this by a reduction from \textrm{DISJ}.
    For simplicity, it suffices to assume $d=2$ and $T=2$ in the \textrm{VKMC} problem.
    Given a \textrm{DISJ} instance of size $n$, let $a \in \{0,1\}^{n}$ be Alice's input and $b \in \{0,1\}^{n}$ be Bob's input.
    They construct an input $\mX \subset \R^2$ for \textrm{VKMC}, where $\mX = \{\vx_{i}: \vx_{i}=(a_{i}, b_{i}), i \in [n]\}$.
    We denote $S \subseteq [n]$ with a weight function $w: S \rightarrow \R_{\geq 0}$ to be an $\eps$-coreset such that for any $\mC \in \gC$ with $|\mC|=k$, we have
    \[
    \costC(S,\mC) := \sum_{i\in S} w(i) \cdot d(\vx_i, \mC)^2 \in (1\pm \eps)\cdot \costC(\mX,\mC).
    \]
    Based on the above guarantee, w.l.o.g., if we set $k=3$ and $\mC=\{(0,0),(0,1),(1,0)\}$, then only point $(1,1)$ can induce positive cost.
    In other words, $\textrm{DISJ}(a,b)=1$ if and only if the set $\{\vx_i: i\in S\}$ includes point $(1,1)$.
    Thus, any $\delta$-error protocol for \textrm{VKMC} coreset construction can be used as a $\delta$-error protocol for \textrm{DISJ}.
    The lower bound follows from Lemma~\ref{lemma:comm_disj}.
\end{proof}

\subsection{Proof of Theorem~\ref{thm:coreset-vkmc}}
Algorithm~\ref{alg:coreset-vkmc} applies the meta Algorithm~\ref{alg:dis} after computing $\{g_i^{(j)}\}$ locally.
The key is to construct local sensitivities $g_i^{(j)}$ so that the sum $\sum_{j\in[T]}g_i^{(j)}$ can approximate global sensitivity $g_i$ well, i.e, with both small $\zeta$ and $\gG$ in Theorem~\ref{thm:meta_alg}.
\paragraph{Constructing local sensitivities.} By the local sensitivities $g_i^{(j)}$ defined in Line 10 of Algorithm~\ref{alg:coreset-vkmc}, we have the following lemma that upper bound both $\zeta$ and $\gG$.
\begin{lemma}[\bf{Upper bounding the global sensitivity of \textrm{VKMC} locally}]
    \label{lemma:sensitivity-vkmc}
    Given a dataset $\mX \subset \R^d$ with Assumption~\ref{ass:vkmc}, an $\alpha$-approximation algorithm for \kMeans\ with $\alpha=O(1)$ and integers $k\geq 1$, $T\geq 1$, the local sensitivities $g_i^{(j)}$ in Algorithm~\ref{alg:coreset-vkmc} satisfies that for any $i \in [n]$, $\sup_{\mC\in \gC}\frac{\costC_i(\mX,\mC)}{\costC(\mX,\mC)} \leq 4\tau \sum_{j\in[T]}g_i^{(j)}$, i.e., $\zeta = O(\tau)$.
    Moreover, $\gG:=\sum_{i\in[n],j\in[T]}g_i^{(j)}=O(\alpha kT)$.
\end{lemma}

The proof can be found in Section~\ref{sec:sensitivity-vkmc}, and it is partly modified from the dimension-reduction type argument \citep{varadarajan2012sensitivity}, which upper bounds the total sensitivity of a point set in clustering problem by projecting points onto an optimal solution.
Intuitively, if some party $t$ satisfies Assumption~\ref{ass:vkmc}, the partition over $[n]$ corresponding to an $\alpha$-approximation computed using local data will induce a global $\alpha\tau$-approximate solution.
Hence, combining this with the argument mentioned above, we derive that $g_i^{(t)}$ (scaled by $4\tau$) is an upper bound of the global sensitivity.
Though unaware of which party satisfies Assumption~\ref{ass:vkmc}, it suffices to sum up $g_i^{(j)}$ over $j\in[T]$, costing an addtional $T$ in $\gG$.

\begin{proof}[Proof of Theorem~\ref{thm:coreset-vkmc}]
    By Lemma~\ref{lemma:sensitivity-vkmc}, the sensitivity gap $\zeta$ is $O(\tau)$ and the total sensitivity $\gG$ is $O(\alpha kT)$.
    Plugging them into Theorem~\ref{thm:meta_alg} completes the proof.
\end{proof}

\subsection{Proof of Lemma~\ref{lemma:sensitivity-vkmc}}
\label{sec:sensitivity-vkmc}
Our proof is partly inspired by \cite{varadarajan2012sensitivity}.
For preparation, we first introduce the following useful notations.

Suppose the party $t$ in the dataset $\mX$ satisfies Assumption~\ref{ass:vkmc}, and $\gA$ is an $\alpha$-approximation algorithm for \kMeans\ clustering.
Let $\Tilde{\mC}^{(t)}$ be an $\alpha$-approximate solution computed locally in party $t$ using $\gA$, i.e., $\Tilde{\mC}^{(t)}=\gA(\mX^{(t)})=\{\Tilde{\vc}_l^{(t)}:l\in[k]\}$.
We define a mapping $\pi:[n]\rightarrow[k]$ to find the closest center index for each point in the local solution, i.e., $\pi(i)=\argmin_{l\in[k]} d(\vx_i^{(t)},\Tilde{\vc}_l^{(t)})$.
We also denote $\mB_l^{(t)}:=\{i\in[n]: \pi(i)=l\}$ to be the local cluster corresponding to $\Tilde{\vc}_l^{(t)}$.
Note that $\{\mB_l^{(t)}: l\in[k]\}$ is a partition over data as $\mB_l^{(t)}\cap\mB_{l'}^{(t)}=\varnothing$ ($l,l'\in [k]$, $l\neq l'$) and $\cup_{l\in[k]}\mB_l^{(t)}=[n]$.
Let $\Tilde{\mC}:=\{\Tilde{\vc}_l: \Tilde{\vc}_l=\frac{1}{|\mB_l^{(t)}|}\sum_{i\in\mB_l^{(t)}}\vx_i\}$ be a $k$-center set in $\R^d$ lifted from $\R^{d_t}$ based on $\{\mB_l^{(t)}\}$.
The following lemma shows that $\Tilde{\mC}$ is also a constant approximation to the global \kMeans\ clustering.

\begin{lemma}[\bf{Local partition induces global constant apporximation for \kMeans }]
\label{lemma:approx-vkmc}
If party $t$ of a dataset $\mX \subset \R^d$ satisfies Assumption~\ref{ass:vkmc}, then given a local $\alpha$-approximate solution $\Tilde{\mC}^{(t)}$, for any $k$-center set $\mC \in \gC$, we have
    \[
    \costC(\mX, \Tilde{\mC}) \leq \tau \costC(\mX^{(t)}, \Tilde{\mC}^{(t)}) \leq \alpha \tau \costC(\mX, \mC).
    \]
Thus, $\Tilde{\mC}$ is an $\alpha\tau$-approximate solution to the global \kMeans\ clustering.
\end{lemma}

\begin{proof}
    \begin{align*}
        \costC(\mX, \Tilde{\mC}) & = \sum_{i=1}^n d(\vx_i, \Tilde{\mC})^2 & \\
        & \leq \sum_{l=1}^{k} \sum_{i\in\mB_l^{(t)}} d(\vx_i, \Tilde{\vc}_l)^2 & (\text{assignment by $\mB_l^{(t)}$ is not optimal})\\
        & = \sum_{l=1}^{k} \frac{1}{2|\mB_l^{(t)}|} \sum_{i,j\in\mB_l^{(t)}} d(\vx_i, \vx_j)^2 & (\text{a standard property of \kMeans\ }) \\
        & \leq \sum_{l=1}^{k} \frac{\tau}{2|\mB_l^{(t)}|} \sum_{i,j\in\mB_l^{(t)}} d(\vx_i^{(t)}, \vx_j^{(t)})^2 & (\text{by Assumption~\ref{ass:vkmc}})\\
        & = \tau \costC(\mX^{(t)}, \Tilde{\mC}^{(t)}) &\\
        & \leq \alpha \tau \costC(\mX^{(t)}, \mC^{(t)}) & (\text{$\Tilde{\mC}^{(t)}$ is $\alpha$-approximation}) \\
        & = \alpha \tau \sum_{i=1}^{n} d(\vx_i^{(t)}, \mC^{(t)})^2 & \\
        & \leq \alpha \tau \sum_{i=1}^{n} d(\vx_i, \mC)^2 & \\
        & = \alpha \tau \costC(\mX, \mC).
    \end{align*}
    Note that $|\Tilde{\mC}|=k$, and the above inequality holds for any $\mC \in \gC$ with $|\mC|=k$.
    Minimizing the last item over $\mC \in \gC$ completes the proof.
\end{proof}

Next, since we get a global constant approximation $\Tilde{\mC}$, we can upper bound the global sensitivities via projecting $\mX$ onto $\Tilde{\mC}$.
Concretely, the following lemma shows that $g_i^{(t)}$ (scaled by $4\tau$) is an upper bound of the global sensitivity of $\vx_i$ if Assumption~\ref{ass:vkmc} holds for party $t$.
\begin{lemma}[\bf{Upper bounding the global sensitivities for \kMeans\ }]
    \label{lemma:g-vkmc}
    If party $t$ of a dataset $\mX \subset \R^d$ satisfies Assumption~\ref{ass:vkmc}, then given a local $\alpha$-approximate solution $\Tilde{\mC}^{(t)}$, we have
        \begin{align}
            \label{eq:g-vkmc}
            \sup_{\mC\in\gC} \frac{d(\vx_i,\mC)^2}{\costC(\mX,\mC)} \leq \frac{4\alpha\tau d(\vx_i^{(t)},\Tilde{\mC}^{(t)})^2}{\costC(\mX^{(t)},\Tilde{\mC}^{(t)})} + \frac{4\alpha\tau \sum_{j\in\mB_{\pi(i)}^{(t)}}d(\vx_j^{(t)},\Tilde{\mC}^{(t)})^2}{|\mB_{\pi(i)}^{(t)}|\costC(\mX^{(t)}, \Tilde{\mC}^{(t)})} + \frac{8\alpha\tau}{|\mB_{\pi(i)}^{(t)}|}.
        \end{align}
\end{lemma}

\begin{proof}
    Let the multi-set $\pi(\mX):=\{\Tilde{\vc}_{\pi(i)}:i\in[n]\}$ be the projection of $\mX$ to $\Tilde{\mC}$.
    We denote $s_{\mX}(\vx_i)$ to be $\sup_{\mC\in\gC}\frac{d(\vx_i,\mC)^2}{\costC(\mX,\mC)}$ for $i\in[n]$.
    Similarly, $s_{\pi(\mX)}(\Tilde{\vc}_l):=\sup_{\mC\in\gC}\frac{d(\Tilde{\vc}_l,\mC)^2}{\costC(\pi(\mX),\mC)}$ for $l\in[\Tilde{k}]$.
    First we show that for any $\mC\in\gG$, the \kMeans\ objective of the multi-set $\pi(\mX)$ w.r.t. $\mC$ can be upper bounded by that of $\mX$ with a constant factor.
    \begin{align*}
        \costC(\pi(\mX),\mC) & = \sum_{i=1}^{n} d(\Tilde{\vc}_{\pi(i)},\mC)^2 &\\
        & = \sum_{i=1}^{n} \min_{l\in[k]} d(\Tilde{\vc}_{\pi(i)}, \vc_l)^2 &\\
        & \leq \sum_{i=1}^{n} \min_{l\in[k]} \left( 2d(\vx_i,\vc_l)^2+2d(\vx_i,\Tilde{\vc}_{\pi(i)})^2 \right) & (\text{triangle inequality for $d^2$}) \\
        & = 2\costC(\mX,\mC)+2\costC(\mX,\Tilde{\mC}) &\\
        & \leq 2(1+\alpha\tau)\costC(\mX,\mC) & (\text{Lemma~\ref{lemma:approx-vkmc}})\\
        & \leq 4\alpha\tau\costC(\mX,\mC). & (\text{$\alpha\tau\geq 1$})
        \numberthis \label{ineq:cost_pi_kmeans}\\
    \end{align*}
    Then for any $\mC\in\gC$ and $\vx_i\in\mX$, we have
    \begin{align*}
        & && d(\vx_i,\mC)^2 & \\
        &= &&\min_{l\in[k]} d(\vx_i,\vc_l)^2 &\\
        & \leq &&\min_{l\in[k]} \left( 2d(\vx_i,\Tilde{\vc}_{\pi(i)})^2+2d(\Tilde{\vc}_{\pi(i)},\vc_l)^2 \right) & (\text{triangle inequality of $d^2$}) \\
        & = &&2d(\vx_i, \Tilde{\vc}_{\pi(i)})^2+2d(\Tilde{\vc}_{\pi(i)},\mC)^2 &\\
        & \leq &&2d(\vx_i, \Tilde{\vc}_{\pi(i)})^2+2s_{\pi(\mX)}(\Tilde{\vc}_{\pi(i)})\costC(\pi(\mX),\mC) & (\text{definition of $s_{\pi(\mX)}$})\\
        & \leq &&2d(\vx_i, \Tilde{\vc}_{\pi(i)})^2+8\alpha\tau s_{\pi(\mX)}(\Tilde{\vc}_{\pi(i)})\costC(\mX,\mC) & (\text{from \eqref{ineq:cost_pi_kmeans}}) \\
        & = &&2d(\vx_i, \frac{1}{|\mB_{\pi(i)}^{(t)}|}\sum_{j\in\mB_{\pi(i)}^{(t)}}\vx_j)^2+8\alpha\tau s_{\pi(\mX)}(\Tilde{\vc}_{\pi(i)})\costC(\mX,\mC) &  (\text{definition of $\Tilde{\mC}$}) \\
        & \leq &&\frac{2}{|\mB_{\pi(i)}^{(t)}|}\sum_{j\in\mB_{\pi(i)}^{(t)}} d(\vx_i,\vx_j)^2 + 8\alpha\tau s_{\pi(\mX)}(\Tilde{\vc}_{\pi(i)})\costC(\mX,\mC) & (\text{convexity of $d^2$}) \\
        & \leq &&\frac{2}{|\mB_{\pi(i)}^{(t)}|}\sum_{j\in\mB_{\pi(i)}^{(t)}} d(\vx_i,\vx_j)^2 + \frac{8\alpha\tau}{|\mB_{\pi(i)}^{(t)}|}\costC(\mX,\mC) & (\text{$s_{\pi(\mX)}(\Tilde{\vc}_{\pi(i)}) \leq \frac{1}{|\mB_{\pi(i)}^{(t)}|}$})\\
        & \leq &&\left( \frac{2}{|\mB_{\pi(i)}^{(t)}|}\sum_{j\in\mB_{\pi(i)}^{(t)}} \frac{d(\vx_i,\vx_j)^2}{\costC(\mX,\mC)} + \frac{8\alpha\tau}{|\mB_{\pi(i)}^{(t)}|} \right) \costC(\mX,\mC). & \\
    \end{align*}

    Thus,
    \begin{align*}
        & &&\frac{d(\vx_i, \mC)^2}{\costC(\mX,\mC)} & \\
        & \leq &&\frac{2}{|\mB_{\pi(i)}^{(t)}|}\sum_{j\in\mB_{\pi(i)}^{(t)}} \frac{d(\vx_i,\vx_j)^2}{\costC(\mX,\mC)} + \frac{8\alpha\tau}{|\mB_{\pi(i)}^{(t)}|} & \\
        & \leq &&\frac{2\tau}{|\mB_{\pi(i)}^{(t)}|}\sum_{j\in\mB_{\pi(i)}^{(t)}} \frac{d(\vx_i^{(t)},\vx_j^{(t)})^2}{\costC(\mX,\mC)} + \frac{8\alpha\tau}{|\mB_{\pi(i)}^{(t)}|} & (\text{Assumption~\ref{ass:vkmc}}) \\
        & \leq &&\frac{2\alpha\tau}{|\mB_{\pi(i)}^{(t)}|}\sum_{j\in\mB_{\pi(i)}^{(t)}} \frac{d(\vx_i^{(t)},\vx_j^{(t)})^2}{\costC(\mX^{(t)},\Tilde{\mC}^{(t)})} + \frac{8\alpha\tau}{|\mB_{\pi(i)}^{(t)}|} & (\text{Lemma~\ref{lemma:approx-vkmc}}) \\
        & \leq &&\frac{4\alpha\tau}{|\mB_{\pi(i)}^{(t)}|}\sum_{j\in\mB_{\pi(i)}^{(t)}} \frac{d(\vx_i^{(t)},\Tilde{\vc}_{\pi(i)})^2+d(\vx_j^{(t)},\Tilde{\vc}_{\pi(i)})^2}{\costC(\mX^{(t)},\Tilde{\mC}^{(t)})} + \frac{8\alpha\tau}{|\mB_{\pi(i)}^{(t)}|} & (\text{triangle inequality of $d^2$}) \\
        & = &&\frac{4\alpha\tau d(\vx_i^{(t)},\Tilde{\mC}^{(t)})^2}{\costC(\mX^{(t)},\Tilde{\mC}^{(t)})} + \frac{4\alpha\tau \sum_{j\in\mB_{\pi(i)}^{(t)}}d(\vx_j^{(t)},\Tilde{\mC}^{(t)})^2}{|\mB_{\pi(i)}^{(t)}|\costC(\mX^{(t)}, \Tilde{\mC}^{(t)})} + \frac{8\alpha\tau}{|\mB_{\pi(i)}^{(t)}|}, &\\
    \end{align*}
    taking supremum over $\mC\in\gC$ completes the proof.
\end{proof}

Now we are ready to prove Lemma~\ref{lemma:sensitivity-vkmc}.

\begin{proof}[Proof of Lemma~\ref{lemma:sensitivity-vkmc}]
    By Lemma~\ref{lemma:g-vkmc}, since some party $t\in[T]$ satisfies Assumption~\ref{ass:vkmc}, then
    \[\sup_{\mC\in \gC}\frac{\costC_i(\mX,\mC)}{\costC(\mX,\mC)} \leq 4\tau g_i^{(t)} \leq 4\tau \sum_{j\in[T]} g_i^{(j)},\]
    where $g_i^{(t)}$ is defined as the right side of \eqref{eq:g-vkmc} for any $t\in[T]$.
    Moreover,
    \begin{align*}
        \gG & = \sum_{i\in[n]} \sum_{j\in[T]} g_i^{(j)} &\\
        & = \sum_{j\in[T]} \sum_{i\in[n]} \left(\frac{\alpha d(\vx_i^{(j)},\Tilde{\mC}^{(j)})^2}{\costC(\mX^{(j)},\Tilde{\mC}^{(j)})} + \frac{\alpha \sum_{i'\in\mB_{\pi(i)}^{(j)}}d(\vx_{i'}^{(j)},\Tilde{\mC}^{(j)})^2}{|\mB_{\pi(i)}^{(j)}|\costC(\mX^{(j)}, \Tilde{\mC}^{(j)})} + \frac{2\alpha}{|\mB_{\pi(i)}^{(j)}|}\right) &\\
        & = \sum_{j\in[T]} \left(\alpha + \alpha + 2k\alpha\right) &\\
        & = 2(k+1)\alpha T.
    \end{align*}
    Hence, $\zeta=O(\tau)$ and $\gG=O(\alpha k T)$, which completes the proof.
\end{proof}

\section{Robust Coresets for \textrm{VRLR} and \textrm{VKMC}}
In this section, we prove that even if the data assumptions~\ref{ass:vrlr} and~\ref{ass:vkmc} fail to  hold, Algorithms~\ref{alg:coreset-vrlr} and~\ref{alg:coreset-vkmc} still provide robust coresets for \textrm{VRLR} (Theorem~\ref{thm:robust-coreset-vrlr}) and \textrm{VKMC} (Theorem~\ref{thm:robust-coreset-vkmc}) in the flavor of approximating with \emph{outliers}.
\subsection{Robust coreset}
In this section, we introduce a general  definition of robust coreset.
For preparation, we first give some notations for a function space, which can be easily specialized to the cases for \textrm{VRLR} and \textrm{VKMC}.
Given a dataset $\mX$ of size $n$, let $F$ be a set of cost functions from $\mX$ to $\R_{\geq 0}$.
For a subset $S \subseteq [n]$ with a weight function $w:S\rightarrow\R_{\geq0}$, we denote $f(S)$ to be the weighted total cost over $S$ for any $f \in F$, i.e., $f(S)=\sum_{i \in S} w(i) f(\vx_i)$.
With a slight abuse of notation, we can see $\mX$ as $[n]$ with unit weight such that $f(\mX)=\sum_{i\in [n]}f(\vx_i)$.
Now we define the \emph{robust coreset} as follows.
\begin{definition}[\bf{Robust coreset}]
    Let $\beta \in [0,1)$, and $\eps \in (0,1)$.
    Given a set $F$ of functions from $\mX$ to $\R_{\geq 0}$, we say that a weighted subset of $S \subseteq [n]$ is a $(\beta,\eps)$-robust coreset of $\mX$ if for any $f\in F$, there exists a subset $O_f \subseteq [n]$ such that
    \begin{align*}
        \frac{|O_f|}{n} \leq \beta, \frac{|S \cap O_f|}{|S|} &\leq \beta, \\
        \left| f(\mX \backslash O_f) - f(S \backslash O_f) \right| &\leq \eps f(\mX).
    \end{align*}
\end{definition}

Roughly speaking, we allow a small portion of data to be treated as outliers and neglected both in $\mX$ and $S$ when considering the quality of $S$.
Note that a $(0,\eps)$-robust coreset is equivalent to a standard $\eps$-coreset, and $S$ provides a slightly weaker approximation guarantee with additive error if $\beta > 0$.
Also note that our definition of robust coreset is a bit different from that in previous work~\citep{feldman2011unified,huang2018epsilon,wang2021robust}, which focus on generating robust coresets from uniform sampling, but basically they all capture similar ideas.
This is because we will be interested in the robustness of importance sampling under the case where a small percentage of data have unbounded sensitivity gap in Algorithm~\ref{alg:dis}, and the above definition gives simpler results.

We propose the following theorem to show that $(S,w)$ returned by Algorithm~\ref{alg:dis} is a $(\beta,\eps)$-robust coreset when size $m$ is large enough.
\begin{theorem}[\bf{The robustness of Algorighm~\ref{alg:dis}}]
    \label{thm:robust_dis}
    Let $\beta$,$\eps\in(0,1)$.
    Given a dataset $\mX$ of size $n$ and a set $F$ of functions from $\mX$ to $\R_{\geq 0}$, let $g_i=\sum_{j\in[T]}g_i^{(j)}$ and $\gG=\sum_{i\in[n]}g_i$.
    Let $S \subseteq [n]$ be a sample of size $m$ drawn i.i.d from $[n]$ with probability proportional to $\{g_i:i\in[n]\}$, where each sample $i\in[n]$ is selected with probability $\frac{g_i}{\gG}$ and has weight $w(i):=\frac{\gG}{m g_i}$.
    If $\forall i\in[n]$, $ j\in[T]$ we have $g_i^{(j)} \geq 1/n$,
    let $s_i:=\sup_{f\in F}\frac{f(\vx_i)}{f(\mX)}$ and $c=  \frac{2\sum_{i\in[n]}s_i}{\beta T}$.
    If
    \begin{align}\label{eq:robust_size}
        m = O\left( \frac{c^2\gG^2}{\eps^2}\left(\dim(F) + \log{\frac{1}{\delta}} \right) \right),
    \end{align}
    where $\dim(F)$ is the pseudo-demension of $F$.
    Then with probability $1-\delta$, $(S, w)$ is a $(\beta, \eps)$-robust coreset of $\mX$.
\end{theorem}

The proof is in Section~\ref{sec:proof_robust_dis}.
Recall that the term $\frac{s_i}{g_i}$ represents the sensitivity gap of point $\vx_i$, and Algorithm~\ref{alg:dis} guarantees sublinear communication complexity only if the maximum sensitivity gap $\zeta$ over all points is independent of $n$.
The main idea in the above theorem is that we can reduce the portion of potential outliers (with large sensitivity gap) to a small constant both in $\mX$ and $S$ via scaling sample size $m$ by a sufficiently large constant.

\subsection{Robust coresets for \textrm{VRLR}}
The following theorem shows that Algorithm~\ref{alg:coreset-vrlr} returns a robust coreset for \textrm{VRLR} when sample size $m$ is large enough.
Note that $m$ is still independent of $n$.
\begin{theorem}[\bf{Robust coresets for \textrm{VRLR}}]\label{thm:robust-coreset-vrlr}
    For a given dataset $\mX \subset \R^d$, integer $T \geq 1$ and constants $\beta, \eps, \delta \in (0,1)$, with probability at least $1-\delta$, Algorithm~\ref{alg:coreset-vrlr} constructs a $(\beta,\eps)$-robust coreset for \textrm{VRLR} of size
    \[
    m = O\left( \frac{d^4}{\eps^2\beta^2T^2}\left(d^2+\log{\frac{1}{\delta}}\right) \right),
    \]
    and uses communication complexity $O(mT)$.
\end{theorem}
\begin{proof}
    By Theorem~\ref{thm:robust_dis}, in \textrm{VRLR}, $F=\{f_{\vtheta}:f_{\vtheta}(\vx)=(\vx^\top \vtheta - \vy)^2 + R(\vtheta)/n, \vtheta \in \R^d \}$.
    Note that in Theorem~\ref{thm:coreset-vrlr},  $g_i^{(j)}=\|\vu_i^{(j)}\|^2+\frac{1}{n} \geq \frac{1}{n}$, $\gG=O(d)$ and $\sum_{i\in[n]}s_i=O(d)$, we have $c\gG = O(\frac{d^2}{\beta T})$.
    Plugging $c\gG=O(\frac{d^2}{\beta T})$ and $\dim(F)=d^2$ into \eqref{eq:robust_size} completes the proof.
\end{proof}

\subsection{Robust coresets for \textrm{VKMC}}
The following theorem shows that Algorithm~\ref{alg:coreset-vkmc} returns a robust coreset for \textrm{VKMC} when sample size $m$ is large enough.
Note that $m$ is still independent of $n$.
\begin{theorem}[\bf{Robust coresets for \textrm{VKMC}}]\label{thm:robust-coreset-vkmc}
    For a given dataset $\mX \subset \R^d$, an $\alpha$-approximation algorithm for \kMeans\ with $\alpha=O(1)$, integers $k\geq 1$, $T\geq 1$ and constants $\beta, \eps, \delta \in (0,1)$, with probability at least $1-\delta$, Algorithm~\ref{alg:coreset-vkmc} constructs a $(\beta,\eps)$-robust coreset for \textrm{VKMC} of size
    \[
    m=O\left(\frac{\alpha^2 k^4}{\eps^2\beta^2}\left(dk+\log{\frac{1}{\delta}}\right)\right),
    \]
    and uses communication complexity $O(mT)$.
\end{theorem}
\begin{proof}
    By Theorem~\ref{thm:robust_dis}, in \textrm{VKMC},
    $F=\{f_{\mC}:f_{\mC}(\vx)=d(\vx,\mC)^2=\min_{\vc\in\mC}d(\vx,c)^2, \mC\in\gC, |\mC|=k\}$.
    Note that in Theorem~\ref{thm:coreset-vkmc},
    $g_i^{(j)}\geq \frac{1}{n}$, $\gG=O(\alpha kT)$ and $\sum_{i\in[n]}s_i=O(k)$, we have $c\gG = O(\frac{\alpha k^2}{\beta})$.
    Plugging $c\gG=O(\frac{d^2}{\beta T})$ and $\dim(F)=dk$ into \eqref{eq:robust_size} completes the proof.
\end{proof}

\subsection{Proof of Theorem~\ref{thm:robust_dis}}
    \label{sec:proof_robust_dis}
    We first introduce the following lemma which mainly shows that importance sampling generates an $\eps$-approximation of $\mX$ on the corresponding weighted function space.
    \begin{lemma}[\bf{Importance sampling on a function space}~\citep{balcan2013distributed,feldman2011unified}]
        \label{lemma:eps-approximation}
        Given a set $F$ of functions from $\mX$ to $\R_{\geq 0}$ and a constant $\eps \in (0,1)$, let $S$ be a sample of size $m$ drawn i.i.d from $[n]$ with probability proportional to $\{g_i:i\in[n]\}$.
        If $g_i=\Omega(\frac{1}{n})$ for any $i\in[n]$, and let $\gG=\sum_{i\in[n]}g_i$.
        If
        \[
        m = O\left(\frac{1}{\eps^2}\left(\dim(F)+\log{\frac{1}{\delta}}\right)\right),
        \]
        where $\dim(F)$ is the pseudo-demension of $F$.
        Then with probability $1-\delta$, $ \forall f\in F$ and $\forall r\geq 0$,
        \[
        \left| \sum_{i\in[n],\frac{f(\vx_i)}{g_i}\leq r} f(\vx_i) - \sum_{i\in S, \frac{f(\vx_i)}{g_i}\leq r} \frac{\gG}{m g_i} f(\vx_i) \right| \leq \gG \eps r.
        \]
    \end{lemma}
    
    Now we are ready to prove Theorem~\ref{thm:robust_dis}.
    \begin{proof}[Proof of Theorem~\ref{thm:robust_dis}]
        Recall that $c = \frac{2\sum_{i\in[n]}s_i}{\beta T}$.
        Let $O \subseteq [n]$ be defined as 
        \[O := \{i\in[n]:s_i\geq c g_i\}.\]
        Note that $g_i = \sum_{j\in[T]}g_i^{(j)} \geq \frac{T}{n}$, and $\sum_{i\in [n]} s_i \geq \sum_{i\in O} s_i \geq \sum_{i\in O} c g_i \geq |O|\cdot \frac{cT}{n}$.
        Hence,
        \begin{align}\label{eq:robust_num}
            \frac{|O|}{n} \leq \frac{\sum_{i\in[n]}s_i}{cT} = \frac{\beta}{2} < \beta.
        \end{align}
        Let $p$ be the probability that a point in $S$ belongs to $O$, then
        \begin{align*}
            p = \frac{\sum_{i\in O} g_i}{\sum_{i\in [n]} g_i} \leq \frac{\sum_{i\in O} s_i}{c\sum_{i\in[n]} g_i} \leq \frac{\sum_{i\in O} s_i}{cT} \leq \frac{\sum_{i\in[n]} s_i}{cT} = \frac{\beta}{2}.
        \end{align*}
        Hence, by a standard multiplicative Chernoff bound, if $m = \Omega(\frac{1}{\beta}\log{1/\delta})$, then with probability $1-\delta/2$, we have
        \begin{align}\label{eq:robust_weight}
            \frac{|S \cap O|}{|S|} \leq \beta.
        \end{align}
        For any $f\in F$, we define a subset $O_f \subseteq O$ as follows,
        \[O_f := \{ i\in[n]: \frac{f(\vx_i)}{f(\mX)} \geq c g_i \}.\]
        By \eqref{eq:robust_num} and \eqref{eq:robust_weight}, we have that $\frac{|O_f|}{n} \leq \beta$ and $\frac{|S \cap O_f|}{|S|} \leq \beta$.
        Note that $f(\vx_i)/g_i \geq cf(\mX)$ if and only if $i\in O_f$. 
        Let $r=cf(\mX)$ and plug it into Lemma~\ref{lemma:eps-approximation}, then
        \begin{align*}
            \left| \sum_{i\in[n],\frac{f(\vx_i)}{g_i}\leq r} f(\vx_i) - \sum_{i\in S, \frac{f(\vx_i)}{g_i}\leq r} \frac{\gG}{m g_i} f(\vx_i) \right|=|f(\mX \backslash O_f)-f(S\backslash O_f)| \leq \gG c\eps f(\mX),
        \end{align*}
        scaling $\eps$ by $\frac{1}{c\gG}$ completes the proof.
    \end{proof}

\end{document}